\def\year{2022}\relax
\documentclass[letterpaper]{article} 
\usepackage{aaai22}  
\usepackage{times}  
\usepackage{helvet}  
\usepackage{courier}  
\usepackage[hyphens]{url}  
\usepackage{graphicx} 
\usepackage{natbib}  
\usepackage{caption} 
\DeclareCaptionStyle{ruled}{labelfont=normalfont,labelsep=colon,strut=off} 
\usepackage[utf8]{inputenc} 
\usepackage[T1]{fontenc}    
\usepackage{hyperref}       
\usepackage{url}            
\usepackage{booktabs}       
\usepackage{amsfonts}       
\usepackage{nicefrac}       
\usepackage{microtype}      
\usepackage{xcolor}         

\usepackage{todonotes}
\usepackage{caption}
\usepackage{subcaption}
\usepackage[ruled,linesnumbered,noend]{algorithm2e}
\usepackage{placeins}
\usepackage{adjustbox}

\usepackage{tikz}
\usepackage[utf8]{inputenc}
\usepackage{pgfplots}
\DeclareUnicodeCharacter{2212}{−}
\usepgfplotslibrary{groupplots,dateplot}
\usetikzlibrary{patterns,shapes.arrows}
\pgfplotsset{compat=newest}

\usepackage{amsmath}
\usepackage{amsthm}
\DeclareMathOperator*{\argmax}{arg\,max}
\DeclareMathOperator*{\argmin}{arg\,min}

\newtheorem{theorem}{Theorem}[section]

\newtheorem{lemma}[theorem]{Lemma}

\title{Out of Distribution Data Detection Using Dropout Bayesian Neural Networks}
\author{
Andre T. Nguyen,\textsuperscript{\rm 1,2,3}
Fred Lu,\textsuperscript{\rm 1,2,3}
Gary Lopez Munoz,\textsuperscript{\rm 1,2}
Edward Raff,\textsuperscript{\rm 1,2,3}
Charles Nicholas,\textsuperscript{\rm 3}
James Holt\textsuperscript{\rm 1}

}

\affiliations{
    \textsuperscript{\rm 1}Laboratory for Physical Sciences\\
    \textsuperscript{\rm 2}Booz Allen Hamilton\\
    \textsuperscript{\rm 3}University of Maryland, Baltimore County\\
    andre@lps.umd.edu, lu\_fred@bah.com, dlmgary@lps.umd.edu, edraff@lps.umd.edu, nicholas@umbc.edu, holt@lps.umd.edu
}

\begin{document}

\maketitle

\begin{abstract}
We explore the utility of information contained within a dropout based Bayesian neural network (BNN) for the task of detecting out of distribution (OOD) data. We first show how previous attempts to leverage the randomized embeddings induced by the intermediate layers of a dropout BNN can fail due to the distance metric used. We introduce an alternative approach to measuring embedding uncertainty, justify its use theoretically, and demonstrate how incorporating embedding uncertainty improves OOD data identification across three tasks: image classification, language classification, and malware detection. 
\end{abstract} 

\section{Introduction}

Detecting out of distribution (OOD) data at test time is critical in a variety of machine learning applications. For example, in the context of malware classification \cite{Raff2020AClassification}, OOD data could correspond to the emergence of a new form of malicious attack. 
\citet{Gal2016DropoutLearning} developed an approach to variational inference in Bayesian neural networks (BNNs) that showed  a neural network with dropout \cite{Hinton2012ImprovingDetectors,Srivastava2014Dropout:Overfitting}, a technique commonly used to reduce overfitting in neural networks (NNs) by randomly dropping units during training, applied before every weight layer is equivalent to an approximation of a deep Gaussian process \cite{Damianou2013DeepProcesses}. Training with dropout effectively performs variational inference for the deep Gaussian process model, and the posterior distribution can be sampled from by leaving dropout on at test time. This approach to Bayesian deep learning has been popular in practice as it is easy to implement and scales well. 

Measures of uncertainty usually are a function of the sampled softmax outputs of such a BNN, for example predictive entropy and mutual information. There is however useful information at every intermediate layer of a dropout BNN. The dropout based approach to Bayesian deep learning suffers, like most variational inference methods, from the tendency to fit an approximation to a local mode instead of to the full posterior because of a lack of representational capacity and because of the directionality of the KL divergence \cite{Smith2018UnderstandingDetection, Wilson2020BayesianGeneralization}. This behavior however allows us to expect the randomized intermediate representation samples in a dropout BNN to be meaningfully related as they are sampled from a local mode. In this paper, we explore how to leverage additional information generated at every layer of the network for the task of OOD data detection at test time. In particular, we interpret the intermediate representation of a data point at a particular layer as a randomized embedding. The embedding is randomized due to the use of dropout at test time.

The idea to use a randomized embedding induced by the intermediate layers of a dropout BNN has been attempted previously, but can fail due to the underlying Euclidean distance metric used in previous work. The use of Euclidean distance does not account for the confounding variability caused by changes in embedding magnitudes. We will theoretically justify and empirically show that by instead using a measure based on cosine distance, this problem can be rectified. We then leverage this improved uncertainty estimation to show better OOD data identification across three highly different tasks to demonstrate the robustness of our approach.

The objective of this paper is not to develop a state-of-the-art approach to OOD data detection, but rather in the context of dropout BNNs to: (1) show how to cheaply improve OOD data detection in systems where a dropout BNN is already deployed, by using intermediate computational results that are already being computed but not fully leveraged, and (2) provide theoretical and practical evidence to highlight why it is valuable to deconflate angular information about embedding dispersion from embedding norm information. Additionally, previous works have evaluated OOD detection by assuming access to a large OOD dataset of similar size to the in distribution dataset. This is an unrealistic assumption as in areas like cyber security where OOD examples are limited and expensive. So, we also examine the effect of small dataset sizes for OOD detection in our experiments.

\section{Related Work}

Two kinds of uncertainty can be distinguished \cite{Kendall2017WhatVision}. Aleatoric uncertainty is caused by inherent noise and stochasticity in the data. More training data will not help to reduce this kind of uncertainty. Epistemic uncertainty on the other hand is caused by a lack of similar training data. In regions lacking training data, different model parameter settings that produce diverse or potentially conflicting predictions can be comparably likely under the posterior. OOD data is expected to have higher uncertainty, epistemic in particular. \citet{Mukhoti2021DeterministicUncertainty} prove that one cannot infer epistemic uncertainty from a deterministic model's softmax entropy, so additional information is needed to estimate epistemic uncertainty.   

Uncertainty modeling using probabilistic embeddings has primarily been used for estimating aleatoric uncertainty \cite{Oh2018ModelingEmbedding, Shi2019ProbabilisticEmbeddings, Chun2021ProbabilisticRetrieval, Chang2020DataRecognition} in tasks such as determining the quality of a test input image. These methods do not easily translate to estimating epistemic uncertainty. For example, \citet{Oh2018ModelingEmbedding} try to apply their method on an epistemic uncertainty estimation task and find that it did not work well for novel classes, and they leave the modeling of epistemic uncertainty as future work. 

The only prior work we are aware of that looks at a randomized embedding approach similar to ours is by \citet{Terhorst2020SER-FIQ:Robustness}, who use dropout at test time to generate a stochastic embedding. They estimate face image quality through the stability of the embedding as measured using Euclidean distance. As we will show, the use of Euclidean distance is problematic as it does not account for factors affecting embedding norms and more generally, the assumptions made by \citet{Terhorst2020SER-FIQ:Robustness} are not met in reality. We also note that they are actually estimating epistemic uncertainty (see \cite{Oh2018ModelingEmbedding} for an explanation) when test image quality is an inherently aleatoric uncertainty estimation problem. We will show both empirical evidence as well as mathematical grounding as to why our proposed approach, without the addition of any complexity, fixes these issues.  

There is evidence that intermediate layers of a neural network contain information useful for epistemic uncertainty estimation and out of distribution detection. \citet{Postels2020TheActivations} establish a connection between the density of hidden representations and the information-theoretic surprise of observing a specific sample in the setting of a deterministic neural network. In particular, they suggest that the first layers of a neural network should be used to estimate epistemic uncertainty due to feature collapse, a phenomena where out-of-distribution data is mapped to in-distribution feature representations in later layers of a network \cite{vanAmersfoort2020UncertaintyNetwork, Mukhoti2021DeterministicUncertainty}, though they also suggest that OOD data detection can benefit from aggregating uncertainty information from several layers. Our work differs from their work as we are not fitting a density to representations of the training data, increasing the applicability of our approach to situations where fitting and storing a density is not an option for computational or regulatory reasons. 

Other recent work has also looked at uncertainty estimation using a single forward pass of a neural network that has had its intermediate representations regularized to produce good uncertainty estimates \cite{vanAmersfoort2020UncertaintyNetwork, Liu2020SimpleAwareness}. We note that many single forward pass based methods like \cite{Mukhoti2021DeterministicUncertainty, Liu2020SimpleAwareness} require residual based networks in combination with spectral normalization to enforce a bi-Lipschitz inductive bias \cite{Bartlett2018RepresentingOptimization}. While the method of \cite{vanAmersfoort2020UncertaintyNetwork} is not residual network constrained, it requires significant changes to the model and training procedure. While our approach requires multiple forward passes (as is the case with all dropout BNNs), it does not require any modifications to existing dropout BNNs, by only using information that is already being computed within a dropout BNN. 

\cite{Mandelbaum2017Distance-basedClassifiers} propose a confidence score that uses a data embedding derived from the penultimate layer of a neural network. The embedding is achieved using either a distance-based loss or adversarial training. Similarly to other methods, this method requires density estimation, and our work differs as our method does not involve a comparison to nearest neighbors from the training set, which may be difficult to deploy in practice due to both storage and regulatory constraints. 

Many works have investigated OOD data detection in probabilistic contexts. \citet{Ovadia2019CanShift} benchmarks Bayesian deep learning methods in the context of dataset shift and OOD data at test time. \citet{Xiao2020WatTransformers} use epistemic uncertainty to detect OOD language data. \citet{Ren2019LikelihoodDetection} detect OOD data using likelihood ratios in the context of deep generative models and evaluate on OOD genomic sequences. Our work makes a contribution to probabilistic OOD identification by being the first work to systematically investigate the appropriate use of the randomized embeddings induced by the intermediate layers of a dropout BNN.

\section{Methods}

In a supervised setting, suppose a neural network structure with $N$ (non-linearity included) layers $f_i, i \in [1,N]$ where $x_1$ is the input and $x_{N+1}$ is the prediction: $x_{i+1} = f_i(x_i)$.
\citet{Gal2016DropoutLearning} showed that a neural network with dropout \cite{Hinton2012ImprovingDetectors,Srivastava2014Dropout:Overfitting} applied before every weight layer is equivalent to an approximation of a deep Gaussian process \cite{Damianou2013DeepProcesses}, and that training with dropout effectively performs variational inference for the deep Gaussian process model. At test time, the posterior distribution can be sampled from by leaving dropout on. This gives us the network structure:
\begin{equation} \label{eq:re}
x_{i+1} = f_i(\textrm{dropout}(x_i))
\end{equation}

\subsection{Randomized Embeddings}

\paragraph{Computing an Embedding} In the context of a trained dropout Bayesian neural network, we can use the intermediate representations from the various layers (the $x_{i+1}$ in \autoref{eq:re}) as a randomized embedding of a data point. The embedding is randomized as multiple forward passes with dropout on will yield different embedding values. The variation in the embedding values could be used to measure epistemic uncertainty \cite{Oh2018ModelingEmbedding}, allowing for the detection of OOD data and dataset shift. 

\paragraph{Measuring Uncertainty} A datum is embedded to a set of randomized embedding values at each layer. We can compute the maximum pairwise distance between the embeddings for a specific datum at a specific layer. This can be done at each layer in the BNN, giving us a feature for each layer that can then be used for tasks such as OOD identification. \textit{All previous work has used Euclidean distance to compute the pairwise distances, without examining the appropriateness of Euclidean distance for the task}. Part of our contribution is an analysis in \autoref{sec:why_cos} of why Euclidean distance is in fact not appropriate, and we introduce a preferable cosine distance based approach which we use in all of our experiments. A small value of 1e-6 was added to the embeddings to avoid numerical issues caused by corner-case zero normed embedding vectors.\footnote{We also note that normalized Euclidean distance, where embedding vectors are normalized to unit length prior to computing Euclidean distance, could also be used in place of cosine distance as its square can be shown to be proportional to cosine distance.} In our experiments, embeddings from non-linear layers (such as convolutions) are flattened prior to computing this metric. A summary of our approach can be found in \autoref{alg:re}. The intuition behind this approach is that if measured appropriately, the ``spread'' or maximal variation in a datum's embedding contains uncertainty information. If all embedding samples are realized to a same point in the embedding space, then there is less uncertainty than if the embedding samples are realized to wildly different parts of the embedding space.

\begin{algorithm}[tb]
\DontPrintSemicolon
\KwIn{A datum $x$, a $N$ layer NN trained with dropout $\{f_1, ..., f_N\}$, and number of samples $T$.}
\KwOut{$N$ randomized embedding based features $z_1,...,z_N$, each corresponding to a layer in the network, for a OOD data detection task.}
\For{t $\gets$ 1 to T}{  
    \For{i $\gets$ 1 to N}{  
        $x_{i+1,t} \gets f_i(\textrm{dropout}(x_{i,t}))$ 
    }
}
\For{i $\gets$ 1 to N}{ 
        $z_{i} \gets \textrm{max}(\textrm{PairwiseCosineDistances}(x_{i,:}))$ 
    } 
\KwRet{$z_1,...,z_N$} \tcp*{Return features.}

\caption{Computing Randomized Embedding Based Features for OOD Data Detection}
\label{alg:re}
\end{algorithm}

\subsection{Baseline Features}

We compare the addition of our randomized embedding based features to a set of common baseline features. For classification tasks, uncertainty estimates in dropout BNNs are usually a function of the sampled softmax outputs. In particular, overall uncertainty can be measured using predictive distribution entropy: $H[\mathbb{P}(y|x,D)] = - \sum_{y \in C}{\mathbb{P}(y|x,D) \log \mathbb{P}(y|x,D)}$.
To isolate and measure epistemic uncertainty mutual information can be used:
$I(\theta,y|D,x) = H[\mathbb{P}(y|x,D)] - \mathbb{E}_{\mathbb{P}(\theta|D)}H[\mathbb{P}(y|x,\theta)]$.

The terms of these equations can be approximated using Monte Carlo estimates obtained by sampling from the dropout BNN posterior \cite{Smith2018UnderstandingDetection}. In particular, $\mathbb{P}(y|x,D) \approx \frac{1}{T} \sum_{i=1}^T \mathbb{P}(y|x,\theta_i)$ and $\mathbb{E}_{\mathbb{P}(\theta|D)}H[\mathbb{P}(y|x,\theta)] \approx \frac{1}{T} \sum_{i=1}^T H[\mathbb{P}(y|x,\theta_i)]$ where the $\theta_i$ are samples from the posterior over models and $T$ is the number of samples. In addition to predictive distribution entropy and mutual information, we also use maximum softmax probability (the value of the largest element of $\mathbb{P}(y|x,D)$) as a feature, shown by \citet{Hendrycks2016ANetworks} to be an effective baseline for the OOD data detection task.

\subsection{How to Measure Embedding Dispersion}

We will now explore why Euclidean distance as used by previous works is not appropriate to measure randomized embedding dispersion. We illustrate using a LeNet5 \cite{YannLeCun1998Gradient-BasedRecognition} model with added dropout before each layer trained on MNIST, with MNIST variants as OOD data. Further data, model, and experimental details correspond to those expanded upon in \autoref{sec:mnist}.

\subsubsection{The Problem With Euclidean Distance} \label{sec:prob_euc_dis}

\citet{Terhorst2020SER-FIQ:Robustness} suggest the Euclidean distance to measure when a data point is suitable for a downstream task, where lower variability in the stochastic embedding induced by a dropout neural network suggests higher suitability for a data point. In particular, they use the sigmoid of the negative mean Euclidean distance between all stochastic embedding pairs for a data point as the measure of suitability. In other words, their hypothesis is that a form of uncertainty can be measured using the Euclidean distance between embedding samples.

We find that if Euclidean distance is used as the metric to measure distance between samples, their hypothesis holds only with excessive training and likely over-fitting.
\autoref{figs-re/lenet-9-4-reg.png} shows that with enough training to get to the accuracy plateau (10 epochs of training with a batch size of 64, with a test accuracy of 0.9885), we actually see the opposite effect. Embeddings for OOD data are actually less spread out than embeddings for in distribution data. 
\autoref{figs-re/lenet-99-4-reg.png} shows that with excessive training (100 epochs of training, with a lower test accuracy of 0.9882), we see that the hypothesis holds better but note that there is still a good amount of overlap between the histograms, limiting the usefulness for OOD detection (and adding a difficult to select stopping criteria). We note that what we are observing is \textit{not} feature collapse. 

This points to two issues that we need to resolve. First, how can we get consistent behavior regardless of over/under-training? Second, how can we more usefully measure spread in a way that matches intuition?

\subsubsection{Spectral Normalization Stabilizes Behavior}

Spectral normalization rescales the weights during training with the spectral norm of the weight matrix, enforcing a Lipschitz constraint that bounds the derivative of the learned function \cite{Miyato2018SpectralNetworks}. This helps to preserve distance as a data point makes its way through the network. \autoref{figs-re/lenet-99-4-sn.png} shows that a spectral normalized version of the network results in consistent behavior even with longer training (100 epochs of training, with a test accuracy of 0.9927). So, there is a solution to the first problem. However, we still see that the spread for OOD data is lower than for in distribution data. 

\subsubsection{Why Cosine Distance Is Needed To Properly Measure Embedding Dispersion} \label{sec:why_cos}

Previous research around OOD detection has noted that a lower maximal softmax output value is correlated with a data point being OOD \cite{Hendrycks2016ANetworks}. One possible explanation could be logits (softmax inputs) of smaller norm. This would make intuitive sense as potentially, less neurons would activate for OOD data since OOD data would lack the in distribution features the network is looking for. 

The squared Euclidean distance between vectors $\bf u$ and $\bf v$ can be written as, where $\theta$ is the angle between $\bf u$ and $\bf v$:
\begin{equation} \label{eq:euc}
|| {\bf u} - {\bf v} ||^2 = || {\bf u} ||^2 + || {\bf v} ||^2 - 2 \: || {\bf u} || \: || {\bf v} || \: \cos{\theta}
\end{equation}
If embedding norms are inherently smaller for OOD data, then Euclidean distance which is norm dependent cannot be used to compare embedding spread across OOD and in distribution datasets, due to confounding. As shown in \autoref{eq:euc}, angular information is affected by norm in both an additive and multiplicative manner with Euclidean distance. So, assuming confounding caused by systematic norm differences, cosine distance should be used to isolate the angular information when measuring embedding dispersion. If Euclidean distance mostly captures information already captured by the norm, then the benefit of being Bayesian for this task is not fully leveraged as norm can be estimated with a single point estimate. \textit{To take full advantage of a dropout BNN, angular information about embedding dispersion needs to be deconflated from embedding norm information.}

We explored this hypothesis and found it to be empirically true and formally justifiable. In \autoref{figs-re/line-euc.png}, Euclidean distance is used to measure embedding dispersion, we see that dispersion is correlated with the logits norm and that the relationship is nearly identical for OOD and in distribution data. This means that measuring the spread of the embeddings using Euclidean distance conveys little extra information than just looking at the norm of the logits. In Appendix \autoref{sec:appendix_sim_mean_variance}, we perform a simulation to further illustrate this problem in the case of a two layer ReLU activated network.

We want to measure spread in a way that is independent of the embedding norm. This can be done a couple of different ways. For example, a simple switch to cosine distance could be used, or the embeddings could be normalized prior to using Euclidean distance (which can be shown to be related to cosine distance). As illustrated in \autoref{figs-re/line-nomedeuc.png}, using cosine distance results in OOD and in distribution data having behaviors that are no longer identical. Appendix \autoref{sec:appendix_unsup} shows similar results in an unsupervised setting involving a stacked denoising autoencoder variant. 

\autoref{figs-re/lenet-9-4-reg-cos.png} shows the same information as \autoref{figs-re/lenet-9-4-reg.png}, except a cosine distance based measure of spread is used instead of a Euclidean based one. With cosine distance, we now see the expected behavior of OOD having more spread than in distribution, and we see a better separation as well which is good for OOD detection. We have shown results for the last layer of a network but note that a similar analysis can be done for each layer. 
Having shown empirical evidence for why angular information needs to be isolated from norm information when measuring embedding dispersion, we next provide a formal analysis for why cosine distance allows for an additional source of information.  

\begin{figure}[tb]
    \centering
    
    \begin{subfigure}[]{0.5\columnwidth}
        \centering
          \input{figs-re/reg-9-4-euc.tex}
        \caption{LeNet5 with 10 epochs of training, Euclidean based measure of embedding dispersion.} 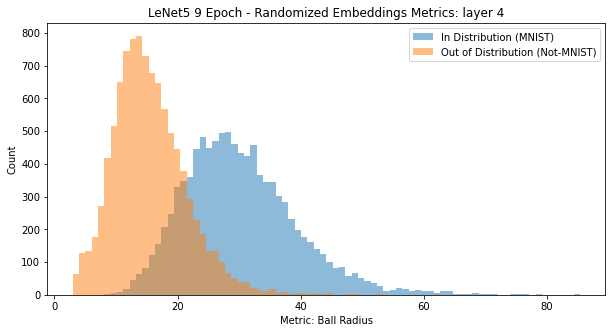
\label{figs-re/lenet-9-4-reg.png}
    \end{subfigure}%
    ~
    \begin{subfigure}[]{0.5\columnwidth}
        \centering
        \input{figs-re/reg-99-4-euc.tex}
        \caption{LeNet5 with 100 epochs of training, Euclidean based measure of embedding dispersion.} 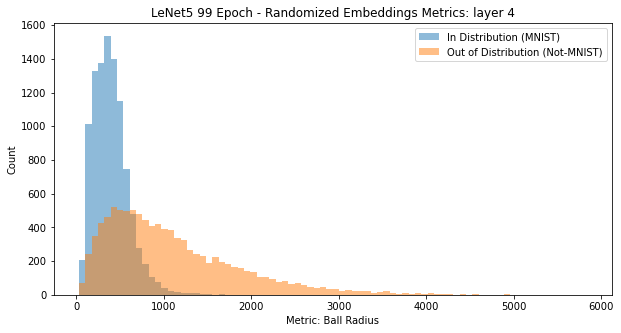
\label{figs-re/lenet-99-4-reg.png}
    \end{subfigure}%
    \\
    \begin{subfigure}[]{0.5\columnwidth}
        \centering
        \input{figs-re/sn-99-4-euc.tex}
        \caption{Spectral Normalized LeNet5 with 100 epochs of training, Euclidean measure of 
        dispersion.} \label{figs-re/lenet-99-4-sn.png}
    \end{subfigure}%
    ~
    \begin{subfigure}[]{0.5\columnwidth}
        \centering
        \input{figs-re/reg-9-4-cos.tex}
        \caption{LeNet5 with 10 epochs of training, cosine based measure of embedding dispersion.} 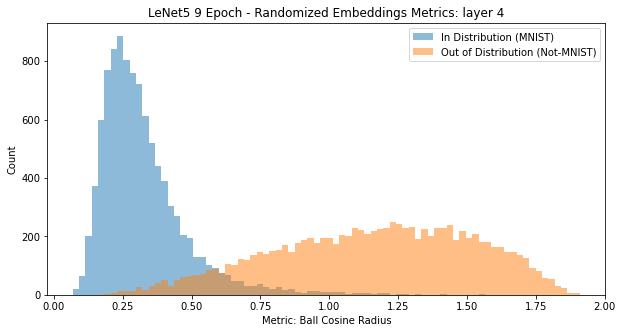
\label{figs-re/lenet-9-4-reg-cos.png}
    \end{subfigure}%

    \caption{Comparison of last layer randomized embedding dispersion distributions for in distribution data (MNIST) and OOD data (Not-MNIST).}
\end{figure}    
    
\begin{figure}[tb]
    \centering
    \begin{subfigure}[]{0.5\columnwidth}
        \centering
        \input{figs-re/line-reg-9-4-euc.tex}
        \caption{Euclidean distance.} 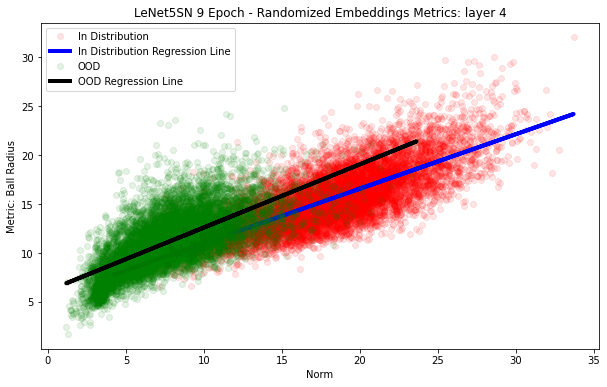
\label{figs-re/line-euc.png}
    \end{subfigure}%
    ~
    \begin{subfigure}[]{0.5\columnwidth}
        \centering
        \input{figs-re/line-reg-9-4-cos.tex}
        \caption{Cosine distance.} 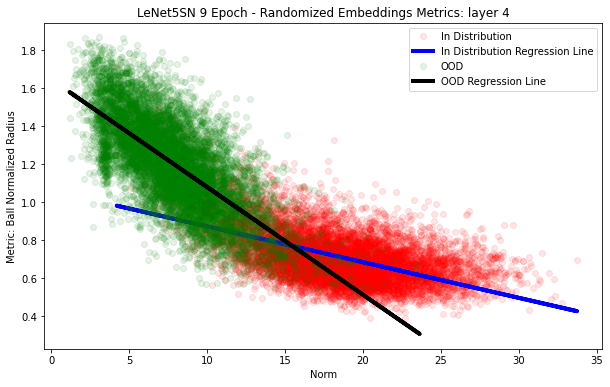
\label{figs-re/line-nomedeuc.png}
    \end{subfigure}%
    
    \caption{A comparison of the relationships between last layer randomized embedding mean norm and the maximum pairwise distance for Euclidean and cosine distances respectively, for in distribution data (MNIST) and OOD data (Not-MNIST). Both models using LeNet5 trained for 10 epochs. 
    }
    
\end{figure}

\subsubsection{Formal Analysis of Cosine Embedding Dispersion} \label{sec:formal_math}

We aim to compute a metric that is invariant to the relative magnitudes among embedding samples, and also accurately represents the dispersion of the embedding samples. In the following, we argue that the mutual information score is not satisfactory for these two objectives. Our goal is not to replace the mutual information as an uncertainty measure, but rather to demonstrate that our pairwise cosine similarity yields an additional source of information that is not captured otherwise.

Let $\{z_i\}_{i=1}^m$ denote $m$ embedding vectors sampled through dropout. The mutual information score is defined as
$$ I(w, y | D, x) = H[p(y|x, D)] - \mathbb{E}_{p(w|D)} H[p(y|x, w)]$$
and is approximated by 
$$ \hat I(w, y|D, x) = H\left[ \frac{1}{m} \sum_{i=1}^m \mathrm{softmax}(z_i) \right] - \frac{1}{m} \sum_{i=1}^m H\left[\mathrm{softmax}(z_i)\right]$$
where $H(\cdot)$ is the entropy function $H(y) = - \sum_i y_i \log y_i$. 

We first introduce a theorem from \citet{Amos2019DifferentiableLearning} that clarifies the geometric properties of the softmax function. The proof is readily shown using Lagrange multipliers.

\begin{theorem}
The softmax function $\mathrm{softmax}(x)_j = \frac{\exp(x_j)}{\sum_i \exp(x_i)}$ is a map from $\mathbb{R}^d$ to the $(d-1)$-simplex that satisfies
$$ \mathrm{softmax}(x) = \argmin_{0 < y < 1} -x^\top y - H(y)\ \ \mathrm{s.t.}\ \ 1^\top y = 1$$
\end{theorem}

From this we see that the softmax solution is a balance between two competing objectives: maximizing $x^\top y$ which aims to place all weight on the coordinate with the largest $x_i$ value, and maximizing the entropy of $y$ which steers toward the uniform vector with value $1/d$. In addition, the softmax temperature changes the relative weighting, which allows us to evaluate the effect of the magnitude of the embedding vector. We leverage this for a further Lemma and Theorem:

\begin{lemma}
The softmax function with temperature $\alpha$, defined by $\mathrm{softmax}(x / \alpha)$, satisfies 
$$ \mathrm{softmax}(x / \alpha) = \argmax_{0 < y < 1} x^\top y + \alpha H(y)\ \ \mathrm{s.t.}\ \ 1^\top y = 1$$
\end{lemma}

\begin{proof}
From the previous theorem we get $\mathrm{softmax}(x/\alpha) = \argmin_{0 < y < 1} -(x / \alpha)^\top y - H(y)\ \ \mathrm{s.t.}\ \ 1^\top y = 1$. Multiplying by scalar $\alpha$ and switching the optimization to maximizing the negative does not change the optimal solution, yielding the statement above. 
\end{proof}

These facts help indicate that softmax-based metrics are not suited for assessing the angular dispersion among vectors. We note that the mapped vector is $\alpha$-dependent and hence dependent on the $L_2$ magnitude of the input vector. Furthermore, arbitrary translations of the vector, which can completely change the direction of the vector, do not impact the softmax. These observations are formalized below.

\begin{theorem}
The softmax function is invariant to translation of input vector $x$. It is not invariant to scaling $x$ except in the special case when $x_1=x_2=\ldots=x_d$. Furthermore, as the magnitude of $x$ increases (without changing direction), the softmax shifts weight to the vertex of the simplex corresponding to the largest coordinate in $x$.
\end{theorem}

\begin{proof}
Invariance to translation follows from observing that $\mathrm{softmax}(x + K) = \exp(x_j + K) / \sum_i \exp(x_j + K) = \exp(x_j) / \sum_i \exp(x_j) = \mathrm{softmax}(x)$.

The dependence on scaling follows from Lemma 1.2. Consider two vectors $x, x'$ such that $x' = x / \alpha$. The value of $\alpha$ adjusts the scale of the $H(y)$ term. Since the $\max x^\top y$ objective aims to shift weight in $y$ to the largest $x$ coordinate and the $\max H(y)$ objective aims to distribute weight evenly, their solutions do not coincide, giving $\mathrm{softmax}(x)$ and $\mathrm{softmax}(x')$ different solutions. In the special case that $x_1 = x_2 = \ldots = x_d$ then $x^\top y$ is constant, so the optimization of $H(y)$ gives the uniform distribution vector. Otherwise, increasing the magnitude of $x'$ is equivalent to sending $\alpha \to 0$, which decreases the contribution of $H(y)$. This causes the solution vector to shift weight to the element with largest value in $x$.
\end{proof}

We confirm this analysis by simulation in Appendix \autoref{sec:appendix_sim_corr}, where we find that our new cosine-based feature adds an orthogonal measure of information that is not captured in previously used measures of uncertainty. 

\section{Experiments and Results} \label{sec:exp_res}

In this section, we evaluate the value of randomized embedding based features across three different OOD data detection tasks in the vision, language, and malware domains. All experiments were implemented in PyTorch \cite{Paszke2019PyTorch:Library}, and neural networks were optimized using Adam with the default recommended settings \cite{Kingma2014Adam:Optimization}. A dropout probability of $p=0.1$ was used, and when sampling from the base neural network models to compute features for OOD detection, 32 samples are used. Experiments were run on an 80 CPU core machine with 512GB of RAM using a single 16GB Tesla P100 GPU. Experiment specific details are described in their respective sections.

We explore the use of two model classes for the OOD detection algorithms. The first model is an L2-regularized logistic regression (LR) with the regularization strength chosen using 3-fold cross-validation. We min-max scaled the input features for the LR model to the range $[0,1]$ based on the training data. The second model is a 500 tree random forest (RF) classifier.  We choose these two models to assess linear vs. non-linear behavior in the OOD detection task. We also explore the effect of varied, small training set sizes for the OOD task in all of our experiments. In many production contexts such as cyber security, examples of OOD data are limited and usually expensive to obtain.

\subsection{Image Classification}
\label{sec:mnist}

For our vision experiments, similarly to the evaluation protocol from \cite{vanAmersfoort2020UncertaintyNetwork, Ren2019LikelihoodDetection, Postels2020TheActivations, Mukhoti2021DeterministicUncertainty} we explore MNIST variants as OOD data. In particular, we train our base model, a LeNet5 \cite{YannLeCun1998Gradient-BasedRecognition} with added dropout before each layer, on MNIST and use Kuzushiji-MNIST \cite{Clanuwat2018DeepLiterature}, notMNIST \cite{Bulatov2011NotMNISTDataset}, and Fashion-MNIST \cite{Xiao2017Fashion-MNIST:Algorithms} as OOD data. When training the downstream OOD data detection algorithms, we train the OOD detector on one of the OOD datasets and test on the other two. For example, we first train a digit classifier on MNIST. Then, we train an OOD data detector that uses randomized embedding based features from the digit classifier to classify MNIST vs. notMNIST. Then we test the OOD data detector on MNIST vs. Kuzushiji-MNIST and Fashion-MNIST.

Due to its importance in practical use, we will test the sample efficiency of the OOD tasks (i.e., how few samples of OOD are needed to detect future OOD data). 
In particular, we evaluate performance, as measured by area under the receiver operating characteristic curve (captures desired data ordering performance) and accuracy (captures desired decision making value), using training datasets consisting of $n$=1000, 100, and just 10 data points from each class (in distribution and OOD). We note that this differs from most previous works which have evaluated by assuming access to a large OOD dataset of similar size to the in distribution dataset, an often unrealistic assumption. Each experiment was run 100 times with random training set samples, where all appropriate data not in the training set is included in the test set,  and we report a mean and standard deviation for each. In all of our experiments, the standard deviations are much smaller than effect sizes, so we report only the means in this section, and standard deviations can be found in appendix \autoref{sec:appendix_std}.

\subsubsection{Detecting OOD Data}
\label{sec:mnist_ood_experiment}

\autoref{tab:mnist-ood-reg99} compares performance with and without the cosine embedding spread features for various experimental configurations and OOD detection models for a dropout LeNet5 trained for 100 epochs. Features labeled as ``Last'' consist of common baseline features computed using softmax output samples from the network (predictive entropy, mutual information, and maximum softmax probability). Features labeled as ``Last+Spread'' consist of these baseline features plus our additional randomized embedding maximum cosine spread features for each layer. 

The inclusion of the additional cosine spread features improves OOD detection performance consistently across datasets, training set sizes, and model types. In limited cases where the ``Spread'' features do not improve the LR model, the RF model with ``Spread'' features performs the best overall, suggesting that the relationship is not necessarily linear. \autoref{tab:mnist-ood-sn99_std} in the Appendix summarizes results from a similar experiment where the base model is a spectral normalized dropout LeNet5 trained for 100 epochs. A comparison of \autoref{tab:mnist-ood-reg99} and Appendix \autoref{tab:mnist-ood-sn99_std} suggests that, while spectral normalization is not required to see an improvement from the inclusion of cosine spread features, spectral normalization does improve OOD detection performance consistently. 

In Appendix \autoref{sec:appendix_more_mnist_experiments}, we further examine the need for a small amount of OOD training data, evaluate Euclidean based spread features, and investigate the feature importances associated with our cosine spread features. 

\begin{table}[tb]
\caption{Performance with and without the cosine randomized embedding spread features for various experimental configurations for a dropout LeNet5 trained on MNIST. Features labeled as ``Last'' consist of common baseline features computed using softmax output samples from the network (predictive entropy, mutual information, and maximum softmax probability). Features labeled as ``Last+Spread'' consist of these baseline features plus our additional randomized embedding maximum cosine spread features for each layer. Each experiment was repeated multiple times, and the mean is reported here while the standard deviation is reported in \autoref{sec:appendix}. Best results are shown in \textbf{bold}.} 
\label{tab:mnist-ood-reg99}
\centering
\adjustbox{max width=0.99\columnwidth}{
\begin{tabular}{@{}llll|rrrrrr@{}}
\toprule
  \multicolumn{3}{c}{OOD} &  \multicolumn{1}{l|}{Num/Class} & \multicolumn{2}{c}{n=1000}    & \multicolumn{2}{c}{n=100}    & \multicolumn{2}{c}{n=10}    \\ \cmidrule(lr){1-3} \cmidrule(lr){5-6} \cmidrule(lr){7-8} \cmidrule(lr){9-10}
  &  &  & \multicolumn{1}{l|}{Metric} & \multicolumn{1}{l}{AUC} & \multicolumn{1}{l}{Acc} & \multicolumn{1}{l}{AUC} & \multicolumn{1}{l}{Acc} & \multicolumn{1}{l}{AUC} & \multicolumn{1}{l}{Acc} \\
Train & Test & Model & \multicolumn{1}{l|}{Features} &  &  &  &  &  &  \\ \midrule
Fashion & Kuzushiji & LR & Last & 0.969 & \textbf{0.914} & 0.967 & 0.909 & 0.963 & 0.884 \\
 &  &  & Last+Spread & \textbf{0.979} & 0.914 & \textbf{0.973} & \textbf{0.911} & \textbf{0.967} & \textbf{0.901} \\ \cmidrule(l){3-10} 
 &  & RF & Last & 0.960 & 0.917 & 0.952 & 0.905 & 0.942 & 0.884 \\
 &  &  & Last+Spread & \textbf{0.979} & \textbf{0.922} & \textbf{0.974} & \textbf{0.921} & \textbf{0.969} & \textbf{0.907} \\ \cmidrule(l){2-10} 
 & notMNIST & LR & Last & 0.966 & 0.912 & 0.965 & 0.909 & 0.960 & 0.879 \\
 &  &  & Last+Spread & \textbf{0.983} & \textbf{0.932} & \textbf{0.979} & \textbf{0.925} & \textbf{0.967} & \textbf{0.892} \\ \cmidrule(l){3-10} 
 &  & RF & Last & 0.959 & 0.920 & 0.950 & 0.903 & 0.938 & 0.880 \\
 &  &  & Last+Spread & \textbf{0.985} & \textbf{0.940} & \textbf{0.976} & \textbf{0.924} & \textbf{0.963} & \textbf{0.901} \\ \midrule
Kuzushiji & Fashion & LR & Last & 0.973 & 0.920 & 0.972 & 0.917 & 0.967 & 0.899 \\
 &  &  & Last+Spread & \textbf{0.989} & \textbf{0.948} & \textbf{0.983} & \textbf{0.937} & \textbf{0.978} & \textbf{0.922} \\ \cmidrule(l){3-10} 
 &  & RF & Last & 0.964 & 0.920 & 0.956 & 0.907 & 0.946 & 0.896 \\
 &  &  & Last+Spread & \textbf{0.986} & \textbf{0.943} & \textbf{0.978} & \textbf{0.931} & \textbf{0.967} & \textbf{0.914} \\ \cmidrule(l){2-10} 
 & notMNIST & LR & Last & 0.967 & 0.914 & 0.965 & 0.910 & 0.960 & 0.886 \\
 &  &  & Last+Spread & \textbf{0.984} & \textbf{0.931} & \textbf{0.975} & \textbf{0.914} & \textbf{0.966} & \textbf{0.888} \\ \cmidrule(l){3-10} 
 &  & RF & Last & 0.960 & 0.922 & 0.950 & 0.904 & 0.938 & 0.888 \\
 &  &  & Last+Spread & \textbf{0.982} & \textbf{0.935} & \textbf{0.971} & \textbf{0.921} & \textbf{0.954} & \textbf{0.896} \\ \midrule
notMNIST & Fashion & LR & Last & 0.966 & 0.911 & 0.957 & 0.906 & 0.959 & 0.893 \\
 &  &  & Last+Spread & \textbf{0.978} & \textbf{0.937} & \textbf{0.969} & \textbf{0.928} & \textbf{0.977} & \textbf{0.925} \\ \cmidrule(l){3-10} 
 &  & RF & Last & 0.960 & 0.910 & 0.955 & 0.904 & 0.946 & 0.887 \\
 &  &  & Last+Spread & \textbf{0.988} & \textbf{0.943} & \textbf{0.982} & \textbf{0.935} & \textbf{0.978} & \textbf{0.920} \\ \cmidrule(l){2-10} 
 & Kuzushiji & LR & Last & 0.960 & \textbf{0.900} & 0.946 & \textbf{0.893} & 0.951 & 0.882 \\
 &  &  & Last+Spread & \textbf{0.966} & 0.893 & \textbf{0.949} & 0.886 & \textbf{0.967} & \textbf{0.902} \\ \cmidrule(l){3-10} 
 &  & RF & Last & 0.956 & \textbf{0.906} & 0.950 & 0.901 & 0.941 & 0.883 \\
 &  &  & Last+Spread & \textbf{0.978} & 0.906 & \textbf{0.973} & \textbf{0.915} & \textbf{0.969} & \textbf{0.906} \\ \bottomrule
\end{tabular}
}
\end{table}

\subsection{Language Classification}

Out of distribution data detection is also of interest in natural language processing, where systems are trained to work on specific languages, and inputs from other languages are considered OOD \cite{Xiao2020WatTransformers}. For these experiments, we train a Char-CNN \cite{Zhang2016Character-levelClassification} with dropout added before every layer to classify languages using the WiLI dataset \cite{Thoma2018TheIdentification}. Training consisted of 50 epochs with a batch size of 128, where the 100 most common characters in the training set (after stripping accents) were used as the vocabulary and each datum was truncated/padded to a length of 200 characters. We train the language classification model to distinguish between French, Spanish, German, English, Italian, and Portuguese text. We use Basque, Polish, Luganda, Finnish, Tongan, and Xhosa as out of distribution languages. All of our in and out of distribution languages are chosen to use the Latin writing system. For the OOD task, training sets consisted of $n$=100, 50, 25, and 10 data points from each class (in distribution and OOD). Each experiment was run 100 times with random training data subsamples, where all languages not trained on are tested on. \autoref{tab:lang_results} shows that the inclusion of our randomized embedding based features consistently improves OOD detection across experimental settings, with average and maximal AUC improvements of 0.06 and 0.15.

\begin{table}[tb]
\caption{Performance with and without the cosine randomized embedding spread features for a Char-CNN with dropout added before every layer trained to classify languages using the WiLI dataset. Standard deviations are reported in \autoref{sec:appendix}, and best results are shown in \textbf{bold}.}
\label{tab:lang_results}
\centering
\adjustbox{max width=0.99\columnwidth}{
\begin{tabular}{@{}llll|rrrrrrrr@{}}
\toprule

  \multicolumn{3}{c}{OOD} &  \multicolumn{1}{l|}{Num/Class} & \multicolumn{2}{c}{n=100}    & \multicolumn{2}{c}{n=50}    &
  \multicolumn{2}{c}{n=25}    &\multicolumn{2}{c}{n=10}    \\ \cmidrule(lr){1-3} \cmidrule(lr){5-6} \cmidrule(lr){7-8} \cmidrule(lr){9-10}  \cmidrule(lr){11-12}
  &  &  & \multicolumn{1}{l|}{Metric} & \multicolumn{1}{l}{AUC} & \multicolumn{1}{l}{Acc} & \multicolumn{1}{l}{AUC} & \multicolumn{1}{l}{Acc} &
  \multicolumn{1}{l}{AUC} & \multicolumn{1}{l}{Acc} &\multicolumn{1}{l}{AUC} & \multicolumn{1}{l}{Acc} \\
 Train & Test & Model & \multicolumn{1}{l|}{Features} &  &  &  &  &  &  &  &  \\ \midrule

Basque & rest & LR & Last & 0.888 & 0.798 & 0.883 & 0.794 & 0.882 & 0.792 & 0.878 & 0.786 \\
 &  &  & Last+Spread & \textbf{0.926} & \textbf{0.843} & \textbf{0.919} & \textbf{0.836} & \textbf{0.921} & \textbf{0.835} & \textbf{0.926} & \textbf{0.828} \\ \cmidrule(l){3-12} 
 &  & RF & Last & 0.862 & 0.797 & 0.857 & 0.793 & 0.851 & 0.792 & 0.842 & 0.789 \\
 &  &  & Last+Spread & \textbf{0.924} & \textbf{0.845} & \textbf{0.920} & \textbf{0.840} & \textbf{0.918} & \textbf{0.835} & \textbf{0.914} & \textbf{0.824} \\ \midrule
Finnish & rest & LR & Last & 0.888 & 0.795 & 0.885 & 0.792 & 0.881 & 0.790 & 0.883 & 0.786 \\
 &  &  & Last+Spread & \textbf{0.910} & \textbf{0.818} & \textbf{0.908} & \textbf{0.818} & \textbf{0.909} & \textbf{0.821} & \textbf{0.910} & \textbf{0.818} \\ \cmidrule(l){3-12} 
 &  & RF & Last & 0.864 & 0.794 & 0.858 & 0.792 & 0.850 & 0.789 & 0.840 & 0.783 \\
 &  &  & Last+Spread & \textbf{0.913} & \textbf{0.821} & \textbf{0.907} & \textbf{0.821} & \textbf{0.905} & \textbf{0.818} & \textbf{0.904} & \textbf{0.814} \\ \midrule
Luganda & rest & LR & Last & 0.891 & 0.806 & 0.889 & 0.803 & 0.887 & 0.800 & 0.881 & 0.794 \\
 &  &  & Last+Spread & \textbf{0.943} & \textbf{0.864} & \textbf{0.939} & \textbf{0.854} & \textbf{0.935} & \textbf{0.847} & \textbf{0.931} & \textbf{0.837} \\ \cmidrule(l){3-12} 
 &  & RF & Last & 0.866 & 0.800 & 0.859 & 0.797 & 0.854 & 0.796 & 0.843 & 0.785 \\
 &  &  & Last+Spread & \textbf{0.936} & \textbf{0.862} & \textbf{0.930} & \textbf{0.852} & \textbf{0.926} & \textbf{0.843} & \textbf{0.921} & \textbf{0.831} \\ \midrule
Polish & rest & LR & Last & 0.900 & 0.824 & 0.897 & 0.821 & 0.891 & 0.816 & 0.887 & 0.812 \\
 &  &  & Last+Spread & \textbf{0.939} & \textbf{0.866} & \textbf{0.938} & \textbf{0.864} & \textbf{0.935} & \textbf{0.860} & \textbf{0.934} & \textbf{0.852} \\ \cmidrule(l){3-12} 
 &  & RF & Last & 0.870 & 0.793 & 0.860 & 0.787 & 0.854 & 0.783 & 0.850 & 0.780 \\
 &  &  & Last+Spread & \textbf{0.937} & \textbf{0.871} & \textbf{0.932} & \textbf{0.863} & \textbf{0.928} & \textbf{0.855} & \textbf{0.922} & \textbf{0.841} \\ \midrule
Tongan & rest & LR & Last & 0.857 & \textbf{0.815} & 0.841 & \textbf{0.811} & 0.815 & 0.800 & 0.791 & 0.771 \\
 &  &  & Last+Spread & \textbf{0.886} & 0.811 & \textbf{0.877} & 0.810 & \textbf{0.884} & \textbf{0.819} & \textbf{0.880} & \textbf{0.813} \\ \cmidrule(l){3-12} 
 &  & RF & Last & 0.765 & 0.699 & 0.766 & 0.695 & 0.769 & 0.684 & 0.785 & 0.701 \\
 &  &  & Last+Spread & \textbf{0.915} & \textbf{0.847} & \textbf{0.913} & \textbf{0.845} & \textbf{0.906} & \textbf{0.836} & \textbf{0.903} & \textbf{0.823} \\ \midrule
Xhosa & rest & LR & Last & 0.894 & 0.807 & 0.891 & 0.804 & 0.886 & 0.800 & 0.879 & 0.794 \\
 &  &  & Last+Spread & \textbf{0.944} & \textbf{0.866} & \textbf{0.940} & \textbf{0.857} & \textbf{0.933} & \textbf{0.846} & \textbf{0.931} & \textbf{0.838} \\ \cmidrule(l){3-12} 
 &  & RF & Last & 0.864 & 0.787 & 0.857 & 0.782 & 0.849 & 0.778 & 0.846 & 0.774 \\
 &  &  & Last+Spread & \textbf{0.939} & \textbf{0.868} & \textbf{0.934} & \textbf{0.860} & \textbf{0.928} & \textbf{0.852} & \textbf{0.921} & \textbf{0.835} \\ \bottomrule
\end{tabular}
}
\end{table}

We note that while OOD data detection is usually treated as a purely binary classification task by most previous work, OOD versus in distribution is a false binary. There are different levels and degrees of how OOD data can be. In the context of language, we can examine the nuances between different flavors of OOD data. While Basque is a language isolate that linguistically does not share any significant similarities to any other languages, Catalan is a Romance language with many linguistic similarities to French and Italian (and Spanish to a lesser extent). While both Basque and Catalan are considered OOD in our setting, we expect good estimates of epistemic uncertainty to capture the property that Catalan is ``less OOD'' than Basque is.  \autoref{fig:figs-re/eus_cat} shows that this desired property is captured by the norm of our randomized embedding features, while the mutual information distributions for Basque and Catalan are nearly indistinguishable. 

\begin{figure}[tb]
    \centering
    \begin{subfigure}[t]{0.5\columnwidth}
        \centering
\begin{tikzpicture}

\definecolor{color0}{rgb}{0.12156862745098,0.466666666666667,0.705882352941177}
\definecolor{color1}{rgb}{1,0.498039215686275,0.0549019607843137}

\begin{axis}[
legend cell align={left},
legend style={fill opacity=0.8, draw opacity=1, text opacity=1, draw=white!80!black, font=\tiny},
tick align=outside,
tick pos=left,
x grid style={white!69.0196078431373!black},
xlabel={Randomized Embed. Feat. Norm},
ylabel={Density},
xmin=-0.1, xmax=2.1,
xtick style={color=black},
y grid style={white!69.0196078431373!black},
ymin=0, ymax=2.541,
ytick style={color=black},
width=\columnwidth,
height=0.75\columnwidth,
]
\draw[draw=none,fill=color0,fill opacity=0.4] (axis cs:-6.93889390390723e-18,0) rectangle (axis cs:0.1,0);
\addlegendimage{ybar,ybar legend,draw=none,fill=color0,fill opacity=0.4};
\addlegendentry{Basque}

\draw[draw=none,fill=color0,fill opacity=0.4] (axis cs:0.1,0) rectangle (axis cs:0.2,0);
\draw[draw=none,fill=color0,fill opacity=0.4] (axis cs:0.2,0) rectangle (axis cs:0.3,0);
\draw[draw=none,fill=color0,fill opacity=0.4] (axis cs:0.3,0) rectangle (axis cs:0.4,0);
\draw[draw=none,fill=color0,fill opacity=0.4] (axis cs:0.4,0) rectangle (axis cs:0.5,0.04);
\draw[draw=none,fill=color0,fill opacity=0.4] (axis cs:0.5,0) rectangle (axis cs:0.6,0.34);
\draw[draw=none,fill=color0,fill opacity=0.4] (axis cs:0.6,0) rectangle (axis cs:0.7,0.44);
\draw[draw=none,fill=color0,fill opacity=0.4] (axis cs:0.7,0) rectangle (axis cs:0.8,0.54);
\draw[draw=none,fill=color0,fill opacity=0.4] (axis cs:0.8,0) rectangle (axis cs:0.9,0.72);
\draw[draw=none,fill=color0,fill opacity=0.4] (axis cs:0.9,0) rectangle (axis cs:1,0.78);
\draw[draw=none,fill=color0,fill opacity=0.4] (axis cs:1,0) rectangle (axis cs:1.1,0.899999999999999);
\draw[draw=none,fill=color0,fill opacity=0.4] (axis cs:1.1,0) rectangle (axis cs:1.2,1.1);
\draw[draw=none,fill=color0,fill opacity=0.4] (axis cs:1.2,0) rectangle (axis cs:1.3,0.780000000000001);
\draw[draw=none,fill=color0,fill opacity=0.4] (axis cs:1.3,0) rectangle (axis cs:1.4,0.919999999999999);
\draw[draw=none,fill=color0,fill opacity=0.4] (axis cs:1.4,0) rectangle (axis cs:1.5,1.04);
\draw[draw=none,fill=color0,fill opacity=0.4] (axis cs:1.5,0) rectangle (axis cs:1.6,0.619999999999999);
\draw[draw=none,fill=color0,fill opacity=0.4] (axis cs:1.6,0) rectangle (axis cs:1.7,0.919999999999999);
\draw[draw=none,fill=color0,fill opacity=0.4] (axis cs:1.7,0) rectangle (axis cs:1.8,0.560000000000001);
\draw[draw=none,fill=color0,fill opacity=0.4] (axis cs:1.8,0) rectangle (axis cs:1.9,0.24);
\draw[draw=none,fill=color0,fill opacity=0.4] (axis cs:1.9,0) rectangle (axis cs:2,0.0600000000000001);
\draw[draw=none,fill=color1,fill opacity=0.4] (axis cs:-6.93889390390723e-18,0) rectangle (axis cs:0.1,0);
\addlegendimage{ybar,ybar legend,draw=none,fill=color1,fill opacity=0.4};
\addlegendentry{Catalan}

\draw[draw=none,fill=color1,fill opacity=0.4] (axis cs:0.1,0) rectangle (axis cs:0.2,0);
\draw[draw=none,fill=color1,fill opacity=0.4] (axis cs:0.2,0) rectangle (axis cs:0.3,0);
\draw[draw=none,fill=color1,fill opacity=0.4] (axis cs:0.3,0) rectangle (axis cs:0.4,0);
\draw[draw=none,fill=color1,fill opacity=0.4] (axis cs:0.4,0) rectangle (axis cs:0.5,0.14);
\draw[draw=none,fill=color1,fill opacity=0.4] (axis cs:0.5,0) rectangle (axis cs:0.6,0.699999999999999);
\draw[draw=none,fill=color1,fill opacity=0.4] (axis cs:0.6,0) rectangle (axis cs:0.7,1.72);
\draw[draw=none,fill=color1,fill opacity=0.4] (axis cs:0.7,0) rectangle (axis cs:0.8,2.42);
\draw[draw=none,fill=color1,fill opacity=0.4] (axis cs:0.8,0) rectangle (axis cs:0.9,1.98);
\draw[draw=none,fill=color1,fill opacity=0.4] (axis cs:0.9,0) rectangle (axis cs:1,1.56);
\draw[draw=none,fill=color1,fill opacity=0.4] (axis cs:1,0) rectangle (axis cs:1.1,0.819999999999999);
\draw[draw=none,fill=color1,fill opacity=0.4] (axis cs:1.1,0) rectangle (axis cs:1.2,0.34);
\draw[draw=none,fill=color1,fill opacity=0.4] (axis cs:1.2,0) rectangle (axis cs:1.3,0.16);
\draw[draw=none,fill=color1,fill opacity=0.4] (axis cs:1.3,0) rectangle (axis cs:1.4,0.0599999999999999);
\draw[draw=none,fill=color1,fill opacity=0.4] (axis cs:1.4,0) rectangle (axis cs:1.5,0.02);
\draw[draw=none,fill=color1,fill opacity=0.4] (axis cs:1.5,0) rectangle (axis cs:1.6,0.02);
\draw[draw=none,fill=color1,fill opacity=0.4] (axis cs:1.6,0) rectangle (axis cs:1.7,0.04);
\draw[draw=none,fill=color1,fill opacity=0.4] (axis cs:1.7,0) rectangle (axis cs:1.8,0.02);
\draw[draw=none,fill=color1,fill opacity=0.4] (axis cs:1.8,0) rectangle (axis cs:1.9,0);
\draw[draw=none,fill=color1,fill opacity=0.4] (axis cs:1.9,0) rectangle (axis cs:2,0);
\end{axis}

\end{tikzpicture}
        \label{fig:figs-re/eus_cat_re.tex}
    \end{subfigure}%
    ~
    \begin{subfigure}[t]{0.5\columnwidth}
        \centering
\begin{tikzpicture}

\definecolor{color0}{rgb}{0.12156862745098,0.466666666666667,0.705882352941177}
\definecolor{color1}{rgb}{1,0.498039215686275,0.0549019607843137}

\begin{axis}[
legend cell align={left},
legend style={fill opacity=0.8, draw opacity=1, text opacity=1, draw=white!80!black, font=\tiny},
tick align=outside,
tick pos=left,
x grid style={white!69.0196078431373!black},
xlabel={Mutual Information},
ylabel={Density},
xmin=-0.0375, xmax=0.7875,
xtick style={color=black},
y grid style={white!69.0196078431373!black},
ymin=0, ymax=4.312,
ytick style={color=black},
width=\columnwidth,
height=0.75\columnwidth,
]
\draw[draw=none,fill=color0,fill opacity=0.4] (axis cs:3.46944695195361e-18,0) rectangle (axis cs:0.0375,1.81333333333333);
\addlegendimage{ybar,ybar legend,draw=none,fill=color0,fill opacity=0.4};
\addlegendentry{Basque}

\draw[draw=none,fill=color0,fill opacity=0.4] (axis cs:0.0375,0) rectangle (axis cs:0.075,1.33333333333333);
\draw[draw=none,fill=color0,fill opacity=0.4] (axis cs:0.075,0) rectangle (axis cs:0.1125,2.13333333333333);
\draw[draw=none,fill=color0,fill opacity=0.4] (axis cs:0.1125,0) rectangle (axis cs:0.15,2.56);
\draw[draw=none,fill=color0,fill opacity=0.4] (axis cs:0.15,0) rectangle (axis cs:0.1875,2.77333333333333);
\draw[draw=none,fill=color0,fill opacity=0.4] (axis cs:0.1875,0) rectangle (axis cs:0.225,3.04);
\draw[draw=none,fill=color0,fill opacity=0.4] (axis cs:0.225,0) rectangle (axis cs:0.2625,3.04);
\draw[draw=none,fill=color0,fill opacity=0.4] (axis cs:0.2625,0) rectangle (axis cs:0.3,2.18666666666667);
\draw[draw=none,fill=color0,fill opacity=0.4] (axis cs:0.3,0) rectangle (axis cs:0.3375,2.61333333333334);
\draw[draw=none,fill=color0,fill opacity=0.4] (axis cs:0.3375,0) rectangle (axis cs:0.375,1.65333333333333);
\draw[draw=none,fill=color0,fill opacity=0.4] (axis cs:0.375,0) rectangle (axis cs:0.4125,1.38666666666667);
\draw[draw=none,fill=color0,fill opacity=0.4] (axis cs:0.4125,0) rectangle (axis cs:0.45,0.906666666666667);
\draw[draw=none,fill=color0,fill opacity=0.4] (axis cs:0.45,0) rectangle (axis cs:0.4875,0.48);
\draw[draw=none,fill=color0,fill opacity=0.4] (axis cs:0.4875,0) rectangle (axis cs:0.525,0.426666666666666);
\draw[draw=none,fill=color0,fill opacity=0.4] (axis cs:0.525,0) rectangle (axis cs:0.5625,0.106666666666667);
\draw[draw=none,fill=color0,fill opacity=0.4] (axis cs:0.5625,0) rectangle (axis cs:0.6,0.16);
\draw[draw=none,fill=color0,fill opacity=0.4] (axis cs:0.6,0) rectangle (axis cs:0.6375,0);
\draw[draw=none,fill=color0,fill opacity=0.4] (axis cs:0.6375,0) rectangle (axis cs:0.675,0.0533333333333334);
\draw[draw=none,fill=color0,fill opacity=0.4] (axis cs:0.675,0) rectangle (axis cs:0.7125,0);
\draw[draw=none,fill=color0,fill opacity=0.4] (axis cs:0.7125,0) rectangle (axis cs:0.75,0);
\draw[draw=none,fill=color1,fill opacity=0.4] (axis cs:3.46944695195361e-18,0) rectangle (axis cs:0.0375,4.10666666666667);
\addlegendimage{ybar,ybar legend,draw=none,fill=color1,fill opacity=0.4};
\addlegendentry{Catalan}

\draw[draw=none,fill=color1,fill opacity=0.4] (axis cs:0.0375,0) rectangle (axis cs:0.075,1.06666666666667);
\draw[draw=none,fill=color1,fill opacity=0.4] (axis cs:0.075,0) rectangle (axis cs:0.1125,2.29333333333333);
\draw[draw=none,fill=color1,fill opacity=0.4] (axis cs:0.1125,0) rectangle (axis cs:0.15,2.61333333333333);
\draw[draw=none,fill=color1,fill opacity=0.4] (axis cs:0.15,0) rectangle (axis cs:0.1875,2.02666666666667);
\draw[draw=none,fill=color1,fill opacity=0.4] (axis cs:0.1875,0) rectangle (axis cs:0.225,2.08);
\draw[draw=none,fill=color1,fill opacity=0.4] (axis cs:0.225,0) rectangle (axis cs:0.2625,1.76);
\draw[draw=none,fill=color1,fill opacity=0.4] (axis cs:0.2625,0) rectangle (axis cs:0.3,1.97333333333333);
\draw[draw=none,fill=color1,fill opacity=0.4] (axis cs:0.3,0) rectangle (axis cs:0.3375,2.50666666666667);
\draw[draw=none,fill=color1,fill opacity=0.4] (axis cs:0.3375,0) rectangle (axis cs:0.375,1.76);
\draw[draw=none,fill=color1,fill opacity=0.4] (axis cs:0.375,0) rectangle (axis cs:0.4125,1.49333333333333);
\draw[draw=none,fill=color1,fill opacity=0.4] (axis cs:0.4125,0) rectangle (axis cs:0.45,1.06666666666667);
\draw[draw=none,fill=color1,fill opacity=0.4] (axis cs:0.45,0) rectangle (axis cs:0.4875,0.853333333333333);
\draw[draw=none,fill=color1,fill opacity=0.4] (axis cs:0.4875,0) rectangle (axis cs:0.525,0.639999999999999);
\draw[draw=none,fill=color1,fill opacity=0.4] (axis cs:0.525,0) rectangle (axis cs:0.5625,0.32);
\draw[draw=none,fill=color1,fill opacity=0.4] (axis cs:0.5625,0) rectangle (axis cs:0.6,0.0533333333333334);
\draw[draw=none,fill=color1,fill opacity=0.4] (axis cs:0.6,0) rectangle (axis cs:0.6375,0.0533333333333334);
\draw[draw=none,fill=color1,fill opacity=0.4] (axis cs:0.6375,0) rectangle (axis cs:0.675,0);
\draw[draw=none,fill=color1,fill opacity=0.4] (axis cs:0.675,0) rectangle (axis cs:0.7125,0);
\draw[draw=none,fill=color1,fill opacity=0.4] (axis cs:0.7125,0) rectangle (axis cs:0.75,0);
\end{axis}

\end{tikzpicture}
        \label{fig:figs-re/eus_cat_mi.tex}
    \end{subfigure}
    \caption{Basque and Catalan are linguistically similar but different languages. Our cosine based embeddings (left) show that they have high overlap but are more OOD than normal data. Prior work using MI (right) is unable to meaningfully distinguish any difference between the related languages. }
    \label{fig:figs-re/eus_cat}
\end{figure}
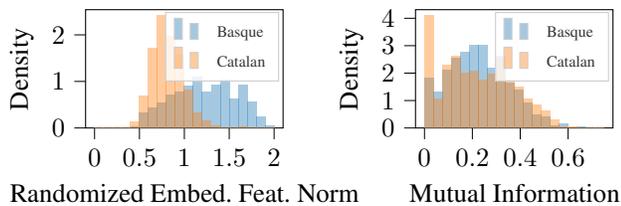

\subsection{Malware Detection}
\label{sec:malware_detection}

Finally, we evaluate the usefulness of our randomized embedding based features in the context of malware detection. Uncovering new or significantly different malware is of particular interest in the quickly evolving cyber security space. We use a dropout variant of the MalConv model \cite{Raff2017MalwareEXE}, a convolutional NN for malware detection that operates on raw byte sequences. We apply dropout before each fully connected layer of MalConv. Applying dropout to only the last layers of a NN corresponds to using maximum a posteriori (MAP) estimates for the initial layers and Bayesian estimates for the later layers \cite{Gal2016DropoutAppendix}. We train the dropout MalConv model for 5 epochs with a batch size of 32 on the EMBER2018 dataset which consists of portable executable files (PE files) scanned by VirusTotal in or before 2018 \cite{Anderson2018EMBER:Models}.

We run two experiments on the Bayesian MalConv model. First, of the 200000 files in the EMBER test set, 363 have as their top most likely malware family label (as labeled by AVClass \cite{Sebastian2016AVCLASS:Labeling}) a family that was not present in the train set. We evaluate OOD detection performance first on these unseen malware families. Second, we evaluate OOD detection performance on a different malware dataset containing malware samples obtained from a Brazilian financial entity \cite{Ceschin2019TheClassifers}. The malware from this dataset could be considered as OOD due to differing geographical specificity and intent, leading to the use of malware tactics, techniques, and procedures likely specific to a Brazilian banking target. There are also temporal differences as the Brazilian samples were all collected before the EMBER dataset, and we additionally only used malware first seen by VirusTotal before 2012. OOD task training sets consisted of $n$=100, 50, and just 25 data points from each class (in distribution and OOD). Each experiment was run 100 times with random train/test splits, where all of the data not in the training set is included in the test set. Results are summarized in \autoref{tab:malware}, showing that the inclusion of our randomized embedding based features consistently improves OOD detection across experimental settings. Because of the high class imbalance in this use case, as access to good OOD data is more limited in the malware domain, we reported the ROC AUC and the recall for the OOD class in \autoref{tab:malware}, noting that recall is often the primary metric of interest in practice for cyber security. 

\begin{table}[tb]
\caption{Performance with and without the cosine randomized embedding spread features for a MalConv model with dropout added before each fully connected layer trained to detect malware using EMBER2018. Standard devs. are reported in \autoref{sec:appendix}, and best results are \textbf{bolded}.}
\label{tab:malware}
\centering
\adjustbox{max width=0.99\columnwidth}{
\begin{tabular}{@{}lll|rrrrrr@{}}
\toprule

  \multicolumn{2}{c}{OOD} &  \multicolumn{1}{l|}{Num/Class} & \multicolumn{2}{c}{n=100}    & \multicolumn{2}{c}{n=50}    & \multicolumn{2}{c}{n=25}    \\ \cmidrule(lr){1-2} \cmidrule(lr){4-5} \cmidrule(lr){6-7} \cmidrule(lr){8-9}
  &   & \multicolumn{1}{l|}{Metric} & \multicolumn{1}{l}{AUC} & \multicolumn{1}{l}{Recall} & \multicolumn{1}{l}{AUC} & \multicolumn{1}{l}{Recall} & \multicolumn{1}{l}{AUC} & \multicolumn{1}{l}{Recall} \\
Experiment & Model & \multicolumn{1}{l|}{Features} &  &  &  &  &  &  \\

\midrule
EMBER2018 & LR & Last & 0.789 & 0.704 & \textbf{0.786} & 0.682 & \textbf{0.778} & 0.650 \\
 &  & Last+Spread & \textbf{0.793} & \textbf{0.718} & 0.783 & \textbf{0.689} & 0.766 & \textbf{0.658} \\ \cmidrule(l){2-9} 
 & RF & Last & 0.757 & 0.735 & 0.752 & 0.727 & 0.748 & 0.714 \\
 &  & Last+Spread & \textbf{0.791} & \textbf{0.784} & \textbf{0.782} & \textbf{0.764} & \textbf{0.770} & \textbf{0.743} \\ \midrule
Brazilian & LR & Last & 0.685 & \textbf{0.645} & 0.680 & 0.607 & 0.668 & 0.584 \\
 &  & Last+Spread & \textbf{0.741} & 0.620 & \textbf{0.734} & \textbf{0.617} & \textbf{0.712} & \textbf{0.605} \\ \cmidrule(l){2-9} 
 & RF & Last & 0.724 & 0.693 & 0.705 & 0.674 & 0.679 & 0.652 \\
 &  & Last+Spread & \textbf{0.839} & \textbf{0.797} & \textbf{0.813} & \textbf{0.772} & \textbf{0.776} & \textbf{0.736} \\ \bottomrule
\end{tabular}
}
\end{table}

\section{Conclusions}

We have demonstrated why previous attempts at measuring randomized embedding dispersion using Euclidean distance are inherently flawed. Then we introduced and theoretically justified a cosine distance based, lightweight approach to test time OOD data detection in the context of dropout Bayesian neural networks. Information that is already computed is used as randomized embeddings, training dataset information does not need to be stored, additional regularization methods are not needed (though do help), and auxiliary neural networks do not need to be trained to take advantage of this additional information.
While we note that our approach is limited to dropout BNNs, the popularity of the dropout approximation to BNNs and the existence of previous works exploring the use of stochastic embeddings based on dropout BNNs suggests the applicability of our approach to practice. Our approach can be deployed anywhere a dropout BNN is already deployed with minimal additional overhead. Future work includes the investigation of more elaborate features based off of the randomized embeddings. 

\appendix

\bibliography{Mendeley-Andre}

\clearpage
\onecolumn
\section{Appendix} \label{sec:appendix}
\FloatBarrier
\subsection{Experimental Result Standard Deviations} \label{sec:appendix_std}

We repeated each of our experiments multiple times and computed a mean and standard deviation for each experiment and evaluation metric. In all of our experiments, the standard deviations are much smaller than effect sizes, so we reported only the means in \autoref{sec:exp_res}. Here we report the complete results, which include standard deviations, for all of our experiments. Vision experiment results are summarized in \autoref{tab:mnist-ood-reg99_std}. Language experiment results are summarized in \autoref{tab:lang_results_std}. Malware experiment results are summarized in \autoref{tab:malware_std}.

\begin{table}[!h]
\caption{Performance with and without the cosine randomized embedding spread features for various experimental configurations for a dropout LeNet5 trained on MNIST. Features labeled as ``Last'' consist of common baseline features computed using softmax output samples from the network (predictive entropy, mutual information, and maximum softmax probability). Features labeled as ``Last+Spread'' consist of these baseline features plus our additional randomized embedding maximum cosine spread features for each layer. Each experiment was repeated multiple times, and the mean and standard deviation are reported here. Best results are shown in \textbf{bold}.}
\label{tab:mnist-ood-reg99_std}
\adjustbox{max width=\columnwidth}{
\begin{tabular}{@{}llll|rrrrrrrrrrrr@{}}
\toprule
 &  &  & Num/Class & \multicolumn{4}{l}{n=1000} & \multicolumn{4}{l}{n=100} & \multicolumn{4}{l}{n=10} \\
 &  &  & Metric & \multicolumn{2}{l}{AUC} & \multicolumn{2}{l}{Acc} & \multicolumn{2}{l}{AUC} & \multicolumn{2}{l}{Acc} & \multicolumn{2}{l}{AUC} & \multicolumn{2}{l}{Acc} \\
 &  &  & Statistic & avg & \multicolumn{1}{r|}{std} & avg & \multicolumn{1}{r|}{std} & avg & \multicolumn{1}{r|}{std} & avg & \multicolumn{1}{r|}{std} & avg & \multicolumn{1}{r|}{std} & avg & std \\
OOD Train & OOD Test & OOD Model & Features &  & \multicolumn{1}{r|}{} &  & \multicolumn{1}{r|}{} &  & \multicolumn{1}{r|}{} &  & \multicolumn{1}{r|}{} &  & \multicolumn{1}{r|}{} &  &  \\ \midrule
Fashion & Kuzushiji & LR & Last & 0.969 & \multicolumn{1}{r|}{0.000} & \textbf{0.914} & \multicolumn{1}{r|}{\textbf{0.001}} & 0.967 & \multicolumn{1}{r|}{0.004} & 0.909 & \multicolumn{1}{r|}{0.009} & 0.963 & \multicolumn{1}{r|}{0.024} & 0.884 & 0.023 \\
 &  &  & Last+Spread & \textbf{0.979} & \multicolumn{1}{r|}{\textbf{0.001}} & 0.914 & \multicolumn{1}{r|}{0.003} & \textbf{0.973} & \multicolumn{1}{r|}{\textbf{0.008}} & \textbf{0.911} & \multicolumn{1}{r|}{\textbf{0.013}} & \textbf{0.967} & \multicolumn{1}{r|}{\textbf{0.035}} & \textbf{0.901} & \textbf{0.033} \\ \cmidrule(l){3-16} 
 &  & RF & Last & 0.960 & \multicolumn{1}{r|}{0.001} & 0.917 & \multicolumn{1}{r|}{0.002} & 0.952 & \multicolumn{1}{r|}{0.005} & 0.905 & \multicolumn{1}{r|}{0.009} & 0.942 & \multicolumn{1}{r|}{0.015} & 0.884 & 0.038 \\
 &  &  & Last+Spread & \textbf{0.979} & \multicolumn{1}{r|}{\textbf{0.001}} & \textbf{0.922} & \multicolumn{1}{r|}{\textbf{0.004}} & \textbf{0.974} & \multicolumn{1}{r|}{\textbf{0.003}} & \textbf{0.921} & \multicolumn{1}{r|}{\textbf{0.005}} & \textbf{0.969} & \multicolumn{1}{r|}{\textbf{0.008}} & \textbf{0.907} & \textbf{0.018} \\ \cmidrule(l){2-16} 
 & notMNIST & LR & Last & 0.966 & \multicolumn{1}{r|}{0.001} & 0.912 & \multicolumn{1}{r|}{0.002} & 0.965 & \multicolumn{1}{r|}{0.003} & 0.909 & \multicolumn{1}{r|}{0.011} & 0.960 & \multicolumn{1}{r|}{0.004} & 0.879 & 0.024 \\
 &  &  & Last+Spread & \textbf{0.983} & \multicolumn{1}{r|}{\textbf{0.001}} & \textbf{0.932} & \multicolumn{1}{r|}{\textbf{0.003}} & \textbf{0.979} & \multicolumn{1}{r|}{\textbf{0.005}} & \textbf{0.925} & \multicolumn{1}{r|}{\textbf{0.015}} & \textbf{0.967} & \multicolumn{1}{r|}{\textbf{0.011}} & \textbf{0.892} & \textbf{0.021} \\ \cmidrule(l){3-16} 
 &  & RF & Last & 0.959 & \multicolumn{1}{r|}{0.002} & 0.920 & \multicolumn{1}{r|}{0.004} & 0.950 & \multicolumn{1}{r|}{0.005} & 0.903 & \multicolumn{1}{r|}{0.010} & 0.938 & \multicolumn{1}{r|}{0.014} & 0.880 & 0.038 \\
 &  &  & Last+Spread & \textbf{0.985} & \multicolumn{1}{r|}{\textbf{0.001}} & \textbf{0.940} & \multicolumn{1}{r|}{\textbf{0.004}} & \textbf{0.976} & \multicolumn{1}{r|}{\textbf{0.004}} & \textbf{0.924} & \multicolumn{1}{r|}{\textbf{0.006}} & \textbf{0.963} & \multicolumn{1}{r|}{\textbf{0.009}} & \textbf{0.901} & \textbf{0.020} \\ \midrule
Kuzushiji & Fashion & LR & Last & 0.973 & \multicolumn{1}{r|}{0.000} & 0.920 & \multicolumn{1}{r|}{0.001} & 0.972 & \multicolumn{1}{r|}{0.001} & 0.917 & \multicolumn{1}{r|}{0.006} & 0.967 & \multicolumn{1}{r|}{0.011} & 0.899 & 0.016 \\
 &  &  & Last+Spread & \textbf{0.989} & \multicolumn{1}{r|}{\textbf{0.001}} & \textbf{0.948} & \multicolumn{1}{r|}{\textbf{0.002}} & \textbf{0.983} & \multicolumn{1}{r|}{\textbf{0.004}} & \textbf{0.937} & \multicolumn{1}{r|}{\textbf{0.008}} & \textbf{0.978} & \multicolumn{1}{r|}{\textbf{0.004}} & \textbf{0.922} & \textbf{0.014} \\ \cmidrule(l){3-16} 
 &  & RF & Last & 0.964 & \multicolumn{1}{r|}{0.001} & 0.920 & \multicolumn{1}{r|}{0.002} & 0.956 & \multicolumn{1}{r|}{0.005} & 0.907 & \multicolumn{1}{r|}{0.009} & 0.946 & \multicolumn{1}{r|}{0.015} & 0.896 & 0.025 \\
 &  &  & Last+Spread & \textbf{0.986} & \multicolumn{1}{r|}{\textbf{0.001}} & \textbf{0.943} & \multicolumn{1}{r|}{\textbf{0.002}} & \textbf{0.978} & \multicolumn{1}{r|}{\textbf{0.004}} & \textbf{0.931} & \multicolumn{1}{r|}{\textbf{0.006}} & \textbf{0.967} & \multicolumn{1}{r|}{\textbf{0.009}} & \textbf{0.914} & \textbf{0.016} \\ \cmidrule(l){2-16} 
 & notMNIST & LR & Last & 0.967 & \multicolumn{1}{r|}{0.000} & 0.914 & \multicolumn{1}{r|}{0.001} & 0.965 & \multicolumn{1}{r|}{0.002} & 0.910 & \multicolumn{1}{r|}{0.010} & 0.960 & \multicolumn{1}{r|}{0.005} & 0.886 & 0.020 \\
 &  &  & Last+Spread & \textbf{0.984} & \multicolumn{1}{r|}{\textbf{0.001}} & \textbf{0.931} & \multicolumn{1}{r|}{\textbf{0.004}} & \textbf{0.975} & \multicolumn{1}{r|}{\textbf{0.008}} & \textbf{0.914} & \multicolumn{1}{r|}{\textbf{0.020}} & \textbf{0.966} & \multicolumn{1}{r|}{\textbf{0.007}} & \textbf{0.888} & \textbf{0.021} \\ \cmidrule(l){3-16} 
 &  & RF & Last & 0.960 & \multicolumn{1}{r|}{0.001} & 0.922 & \multicolumn{1}{r|}{0.003} & 0.950 & \multicolumn{1}{r|}{0.005} & 0.904 & \multicolumn{1}{r|}{0.010} & 0.938 & \multicolumn{1}{r|}{0.017} & 0.888 & 0.026 \\
 &  &  & Last+Spread & \textbf{0.982} & \multicolumn{1}{r|}{\textbf{0.001}} & \textbf{0.935} & \multicolumn{1}{r|}{\textbf{0.003}} & \textbf{0.971} & \multicolumn{1}{r|}{\textbf{0.006}} & \textbf{0.921} & \multicolumn{1}{r|}{\textbf{0.006}} & \textbf{0.954} & \multicolumn{1}{r|}{\textbf{0.010}} & \textbf{0.896} & \textbf{0.019} \\ \midrule
notMNIST & Fashion & LR & Last & 0.966 & \multicolumn{1}{r|}{0.003} & 0.911 & \multicolumn{1}{r|}{0.003} & 0.957 & \multicolumn{1}{r|}{0.018} & 0.906 & \multicolumn{1}{r|}{0.012} & 0.959 & \multicolumn{1}{r|}{0.037} & 0.893 & 0.023 \\
 &  &  & Last+Spread & \textbf{0.978} & \multicolumn{1}{r|}{\textbf{0.003}} & \textbf{0.937} & \multicolumn{1}{r|}{\textbf{0.005}} & \textbf{0.969} & \multicolumn{1}{r|}{\textbf{0.016}} & \textbf{0.928} & \multicolumn{1}{r|}{\textbf{0.015}} & \textbf{0.977} & \multicolumn{1}{r|}{\textbf{0.014}} & \textbf{0.925} & \textbf{0.018} \\ \cmidrule(l){3-16} 
 &  & RF & Last & 0.960 & \multicolumn{1}{r|}{0.002} & 0.910 & \multicolumn{1}{r|}{0.004} & 0.955 & \multicolumn{1}{r|}{0.005} & 0.904 & \multicolumn{1}{r|}{0.011} & 0.946 & \multicolumn{1}{r|}{0.018} & 0.887 & 0.032 \\
 &  &  & Last+Spread & \textbf{0.988} & \multicolumn{1}{r|}{\textbf{0.001}} & \textbf{0.943} & \multicolumn{1}{r|}{\textbf{0.005}} & \textbf{0.982} & \multicolumn{1}{r|}{\textbf{0.004}} & \textbf{0.935} & \multicolumn{1}{r|}{\textbf{0.007}} & \textbf{0.978} & \multicolumn{1}{r|}{\textbf{0.006}} & \textbf{0.920} & \textbf{0.017} \\ \cmidrule(l){2-16} 
 & Kuzushiji & LR & Last & 0.960 & \multicolumn{1}{r|}{0.006} & \textbf{0.900} & \multicolumn{1}{r|}{\textbf{0.007}} & 0.946 & \multicolumn{1}{r|}{0.030} & \textbf{0.893} & \multicolumn{1}{r|}{\textbf{0.021}} & 0.951 & \multicolumn{1}{r|}{0.057} & 0.882 & 0.035 \\
 &  &  & Last+Spread & \textbf{0.966} & \multicolumn{1}{r|}{\textbf{0.006}} & 0.893 & \multicolumn{1}{r|}{0.010} & \textbf{0.949} & \multicolumn{1}{r|}{\textbf{0.028}} & 0.886 & \multicolumn{1}{r|}{0.031} & \textbf{0.967} & \multicolumn{1}{r|}{\textbf{0.030}} & \textbf{0.902} & \textbf{0.035} \\ \cmidrule(l){3-16} 
 &  & RF & Last & 0.956 & \multicolumn{1}{r|}{0.002} & \textbf{0.906} & \multicolumn{1}{r|}{\textbf{0.005}} & 0.950 & \multicolumn{1}{r|}{0.006} & 0.901 & \multicolumn{1}{r|}{0.012} & 0.941 & \multicolumn{1}{r|}{0.020} & 0.883 & 0.033 \\
 &  &  & Last+Spread & \textbf{0.978} & \multicolumn{1}{r|}{\textbf{0.001}} & 0.906 & \multicolumn{1}{r|}{0.008} & \textbf{0.973} & \multicolumn{1}{r|}{\textbf{0.004}} & \textbf{0.915} & \multicolumn{1}{r|}{\textbf{0.012}} & \textbf{0.969} & \multicolumn{1}{r|}{\textbf{0.008}} & \textbf{0.906} & \textbf{0.023} \\ \bottomrule
\end{tabular}
}
\end{table}

\begin{table}[!h]
\caption{Performance with and without the cosine randomized embedding spread features for a Char-CNN with dropout added before every layer trained to classify languages using the WiLI dataset. Best results are shown in \textbf{bold}.}
\label{tab:lang_results_std}
\adjustbox{max width=\columnwidth}{
\begin{tabular}{@{}llll|rrrrrrrrrrrrrrrr@{}}
\toprule
 &  &  & Num/Class & \multicolumn{2}{l}{n=100}  &  &  & \multicolumn{1}{l}{n=50} &  &  &  & \multicolumn{1}{l}{n=25} &  &  &  & \multicolumn{1}{l}{n=10} &  &  &  \\
 &  &  & Metric & \multicolumn{2}{l}{AUC} & \multicolumn{2}{l}{Acc} & \multicolumn{2}{l}{AUC} & \multicolumn{2}{l}{Acc} & \multicolumn{2}{l}{AUC} & \multicolumn{2}{l}{Acc} & \multicolumn{2}{l}{AUC} & \multicolumn{2}{l}{Acc} \\
 &  &  & Statistic & avg & \multicolumn{1}{r|}{std} & avg & \multicolumn{1}{r|}{std} & avg & \multicolumn{1}{r|}{std} & avg & \multicolumn{1}{r|}{std} & avg & \multicolumn{1}{r|}{std} & avg & \multicolumn{1}{r|}{std} & avg & \multicolumn{1}{r|}{std} & avg & std \\
OOD Train & OOD Test & OOD Model & Features &  & \multicolumn{1}{r|}{} &  & \multicolumn{1}{r|}{} &  & \multicolumn{1}{r|}{} &  & \multicolumn{1}{r|}{} &  & \multicolumn{1}{r|}{} &  & \multicolumn{1}{r|}{} &  & \multicolumn{1}{r|}{} &  &  \\ \midrule
Basque & rest & LR & Last & 0.888 & \multicolumn{1}{r|}{0.005} & 0.798 & \multicolumn{1}{r|}{0.008} & 0.883 & \multicolumn{1}{r|}{0.014} & 0.794 & \multicolumn{1}{r|}{0.010} & 0.882 & \multicolumn{1}{r|}{0.015} & 0.792 & \multicolumn{1}{r|}{0.011} & 0.878 & \multicolumn{1}{r|}{0.029} & 0.786 & 0.019 \\
 &  &  & Last+Spread & \textbf{0.926} & \multicolumn{1}{r|}{\textbf{0.019}} & \textbf{0.843} & \multicolumn{1}{r|}{\textbf{0.019}} & \textbf{0.919} & \multicolumn{1}{r|}{\textbf{0.033}} & \textbf{0.836} & \multicolumn{1}{r|}{\textbf{0.023}} & \textbf{0.921} & \multicolumn{1}{r|}{\textbf{0.034}} & \textbf{0.835} & \multicolumn{1}{r|}{\textbf{0.025}} & \textbf{0.926} & \multicolumn{1}{r|}{\textbf{0.026}} & \textbf{0.828} & \textbf{0.023} \\ \cmidrule(l){3-20} 
 &  & RF & Last & 0.862 & \multicolumn{1}{r|}{0.008} & 0.797 & \multicolumn{1}{r|}{0.007} & 0.857 & \multicolumn{1}{r|}{0.013} & 0.793 & \multicolumn{1}{r|}{0.011} & 0.851 & \multicolumn{1}{r|}{0.015} & 0.792 & \multicolumn{1}{r|}{0.013} & 0.842 & \multicolumn{1}{r|}{0.015} & 0.789 & 0.021 \\
 &  &  & Last+Spread & \textbf{0.924} & \multicolumn{1}{r|}{\textbf{0.008}} & \textbf{0.845} & \multicolumn{1}{r|}{\textbf{0.011}} & \textbf{0.920} & \multicolumn{1}{r|}{\textbf{0.011}} & \textbf{0.840} & \multicolumn{1}{r|}{\textbf{0.013}} & \textbf{0.918} & \multicolumn{1}{r|}{\textbf{0.012}} & \textbf{0.835} & \multicolumn{1}{r|}{\textbf{0.017}} & \textbf{0.914} & \multicolumn{1}{r|}{\textbf{0.013}} & \textbf{0.824} & \textbf{0.018} \\ \midrule
Finnish & rest & LR & Last & 0.888 & \multicolumn{1}{r|}{0.003} & 0.795 & \multicolumn{1}{r|}{0.006} & 0.885 & \multicolumn{1}{r|}{0.006} & 0.792 & \multicolumn{1}{r|}{0.008} & 0.881 & \multicolumn{1}{r|}{0.013} & 0.790 & \multicolumn{1}{r|}{0.011} & 0.883 & \multicolumn{1}{r|}{0.008} & 0.786 & 0.015 \\
 &  &  & Last+Spread & \textbf{0.910} & \multicolumn{1}{r|}{\textbf{0.019}} & \textbf{0.818} & \multicolumn{1}{r|}{\textbf{0.017}} & \textbf{0.908} & \multicolumn{1}{r|}{\textbf{0.032}} & \textbf{0.818} & \multicolumn{1}{r|}{\textbf{0.022}} & \textbf{0.909} & \multicolumn{1}{r|}{\textbf{0.035}} & \textbf{0.821} & \multicolumn{1}{r|}{\textbf{0.024}} & \textbf{0.910} & \multicolumn{1}{r|}{\textbf{0.041}} & \textbf{0.818} & \textbf{0.028} \\ \cmidrule(l){3-20} 
 &  & RF & Last & 0.864 & \multicolumn{1}{r|}{0.006} & 0.794 & \multicolumn{1}{r|}{0.007} & 0.858 & \multicolumn{1}{r|}{0.009} & 0.792 & \multicolumn{1}{r|}{0.011} & 0.850 & \multicolumn{1}{r|}{0.014} & 0.789 & \multicolumn{1}{r|}{0.015} & 0.840 & \multicolumn{1}{r|}{0.020} & 0.783 & 0.024 \\
 &  &  & Last+Spread & \textbf{0.913} & \multicolumn{1}{r|}{\textbf{0.011}} & \textbf{0.821} & \multicolumn{1}{r|}{\textbf{0.012}} & \textbf{0.907} & \multicolumn{1}{r|}{\textbf{0.013}} & \textbf{0.821} & \multicolumn{1}{r|}{\textbf{0.012}} & \textbf{0.905} & \multicolumn{1}{r|}{\textbf{0.014}} & \textbf{0.818} & \multicolumn{1}{r|}{\textbf{0.014}} & \textbf{0.904} & \multicolumn{1}{r|}{\textbf{0.016}} & \textbf{0.814} & \textbf{0.015} \\ \midrule
Luganda & rest & LR & Last & 0.891 & \multicolumn{1}{r|}{0.002} & 0.806 & \multicolumn{1}{r|}{0.005} & 0.889 & \multicolumn{1}{r|}{0.004} & 0.803 & \multicolumn{1}{r|}{0.008} & 0.887 & \multicolumn{1}{r|}{0.006} & 0.800 & \multicolumn{1}{r|}{0.009} & 0.881 & \multicolumn{1}{r|}{0.026} & 0.794 & 0.015 \\
 &  &  & Last+Spread & \textbf{0.943} & \multicolumn{1}{r|}{\textbf{0.006}} & \textbf{0.864} & \multicolumn{1}{r|}{\textbf{0.016}} & \textbf{0.939} & \multicolumn{1}{r|}{\textbf{0.008}} & \textbf{0.854} & \multicolumn{1}{r|}{\textbf{0.017}} & \textbf{0.935} & \multicolumn{1}{r|}{\textbf{0.009}} & \textbf{0.847} & \multicolumn{1}{r|}{\textbf{0.014}} & \textbf{0.931} & \multicolumn{1}{r|}{\textbf{0.009}} & \textbf{0.837} & \textbf{0.017} \\ \cmidrule(l){3-20} 
 &  & RF & Last & 0.866 & \multicolumn{1}{r|}{0.006} & 0.800 & \multicolumn{1}{r|}{0.007} & 0.859 & \multicolumn{1}{r|}{0.010} & 0.797 & \multicolumn{1}{r|}{0.011} & 0.854 & \multicolumn{1}{r|}{0.014} & 0.796 & \multicolumn{1}{r|}{0.013} & 0.843 & \multicolumn{1}{r|}{0.027} & 0.785 & 0.038 \\
 &  &  & Last+Spread & \textbf{0.936} & \multicolumn{1}{r|}{\textbf{0.005}} & \textbf{0.862} & \multicolumn{1}{r|}{\textbf{0.009}} & \textbf{0.930} & \multicolumn{1}{r|}{\textbf{0.008}} & \textbf{0.852} & \multicolumn{1}{r|}{\textbf{0.014}} & \textbf{0.926} & \multicolumn{1}{r|}{\textbf{0.009}} & \textbf{0.843} & \multicolumn{1}{r|}{\textbf{0.015}} & \textbf{0.921} & \multicolumn{1}{r|}{\textbf{0.011}} & \textbf{0.831} & \textbf{0.018} \\ \midrule
Polish & rest & LR & Last & 0.900 & \multicolumn{1}{r|}{0.003} & 0.824 & \multicolumn{1}{r|}{0.003} & 0.897 & \multicolumn{1}{r|}{0.010} & 0.821 & \multicolumn{1}{r|}{0.007} & 0.891 & \multicolumn{1}{r|}{0.019} & 0.816 & \multicolumn{1}{r|}{0.012} & 0.887 & \multicolumn{1}{r|}{0.042} & 0.812 & 0.021 \\
 &  &  & Last+Spread & \textbf{0.939} & \multicolumn{1}{r|}{\textbf{0.010}} & \textbf{0.866} & \multicolumn{1}{r|}{\textbf{0.014}} & \textbf{0.938} & \multicolumn{1}{r|}{\textbf{0.014}} & \textbf{0.864} & \multicolumn{1}{r|}{\textbf{0.017}} & \textbf{0.935} & \multicolumn{1}{r|}{\textbf{0.021}} & \textbf{0.860} & \multicolumn{1}{r|}{\textbf{0.023}} & \textbf{0.934} & \multicolumn{1}{r|}{\textbf{0.020}} & \textbf{0.852} & \textbf{0.022} \\ \cmidrule(l){3-20} 
 &  & RF & Last & 0.870 & \multicolumn{1}{r|}{0.010} & 0.793 & \multicolumn{1}{r|}{0.011} & 0.860 & \multicolumn{1}{r|}{0.017} & 0.787 & \multicolumn{1}{r|}{0.015} & 0.854 & \multicolumn{1}{r|}{0.021} & 0.783 & \multicolumn{1}{r|}{0.023} & 0.850 & \multicolumn{1}{r|}{0.029} & 0.780 & 0.034 \\
 &  &  & Last+Spread & \textbf{0.937} & \multicolumn{1}{r|}{\textbf{0.005}} & \textbf{0.871} & \multicolumn{1}{r|}{\textbf{0.008}} & \textbf{0.932} & \multicolumn{1}{r|}{\textbf{0.008}} & \textbf{0.863} & \multicolumn{1}{r|}{\textbf{0.011}} & \textbf{0.928} & \multicolumn{1}{r|}{\textbf{0.009}} & \textbf{0.855} & \multicolumn{1}{r|}{\textbf{0.012}} & \textbf{0.922} & \multicolumn{1}{r|}{\textbf{0.018}} & \textbf{0.841} & \textbf{0.024} \\ \midrule
Tongan & rest & LR & Last & 0.857 & \multicolumn{1}{r|}{0.115} & \textbf{0.815} & \multicolumn{1}{r|}{\textbf{0.060}} & 0.841 & \multicolumn{1}{r|}{0.159} & \textbf{0.811} & \multicolumn{1}{r|}{\textbf{0.076}} & 0.815 & \multicolumn{1}{r|}{0.198} & 0.800 & \multicolumn{1}{r|}{0.088} & 0.791 & \multicolumn{1}{r|}{0.244} & 0.771 & 0.155 \\
 &  &  & Last+Spread & \textbf{0.886} & \multicolumn{1}{r|}{\textbf{0.060}} & 0.811 & \multicolumn{1}{r|}{0.056} & \textbf{0.877} & \multicolumn{1}{r|}{\textbf{0.091}} & 0.810 & \multicolumn{1}{r|}{0.069} & \textbf{0.884} & \multicolumn{1}{r|}{\textbf{0.101}} & \textbf{0.819} & \multicolumn{1}{r|}{\textbf{0.074}} & \textbf{0.880} & \multicolumn{1}{r|}{\textbf{0.125}} & \textbf{0.813} & \textbf{0.085} \\ \cmidrule(l){3-20} 
 &  & RF & Last & 0.765 & \multicolumn{1}{r|}{0.063} & 0.699 & \multicolumn{1}{r|}{0.055} & 0.766 & \multicolumn{1}{r|}{0.074} & 0.695 & \multicolumn{1}{r|}{0.067} & 0.769 & \multicolumn{1}{r|}{0.092} & 0.684 & \multicolumn{1}{r|}{0.075} & 0.785 & \multicolumn{1}{r|}{0.129} & 0.701 & 0.101 \\
 &  &  & Last+Spread & \textbf{0.915} & \multicolumn{1}{r|}{\textbf{0.016}} & \textbf{0.847} & \multicolumn{1}{r|}{\textbf{0.029}} & \textbf{0.913} & \multicolumn{1}{r|}{\textbf{0.019}} & \textbf{0.845} & \multicolumn{1}{r|}{\textbf{0.028}} & \textbf{0.906} & \multicolumn{1}{r|}{\textbf{0.030}} & \textbf{0.836} & \multicolumn{1}{r|}{\textbf{0.035}} & \textbf{0.903} & \multicolumn{1}{r|}{\textbf{0.051}} & \textbf{0.823} & \textbf{0.054} \\ \midrule
Xhosa & rest & LR & Last & 0.894 & \multicolumn{1}{r|}{0.004} & 0.807 & \multicolumn{1}{r|}{0.008} & 0.891 & \multicolumn{1}{r|}{0.007} & 0.804 & \multicolumn{1}{r|}{0.008} & 0.886 & \multicolumn{1}{r|}{0.019} & 0.800 & \multicolumn{1}{r|}{0.013} & 0.879 & \multicolumn{1}{r|}{0.046} & 0.794 & 0.022 \\
 &  &  & Last+Spread & \textbf{0.944} & \multicolumn{1}{r|}{\textbf{0.009}} & \textbf{0.866} & \multicolumn{1}{r|}{\textbf{0.014}} & \textbf{0.940} & \multicolumn{1}{r|}{\textbf{0.014}} & \textbf{0.857} & \multicolumn{1}{r|}{\textbf{0.019}} & \textbf{0.933} & \multicolumn{1}{r|}{\textbf{0.020}} & \textbf{0.846} & \multicolumn{1}{r|}{\textbf{0.024}} & \textbf{0.931} & \multicolumn{1}{r|}{\textbf{0.028}} & \textbf{0.838} & \textbf{0.029} \\ \cmidrule(l){3-20} 
 &  & RF & Last & 0.864 & \multicolumn{1}{r|}{0.009} & 0.787 & \multicolumn{1}{r|}{0.013} & 0.857 & \multicolumn{1}{r|}{0.016} & 0.782 & \multicolumn{1}{r|}{0.020} & 0.849 & \multicolumn{1}{r|}{0.021} & 0.778 & \multicolumn{1}{r|}{0.027} & 0.846 & \multicolumn{1}{r|}{0.028} & 0.774 & 0.032 \\
 &  &  & Last+Spread & \textbf{0.939} & \multicolumn{1}{r|}{\textbf{0.006}} & \textbf{0.868} & \multicolumn{1}{r|}{\textbf{0.011}} & \textbf{0.934} & \multicolumn{1}{r|}{\textbf{0.009}} & \textbf{0.860} & \multicolumn{1}{r|}{\textbf{0.013}} & \textbf{0.928} & \multicolumn{1}{r|}{\textbf{0.012}} & \textbf{0.852} & \multicolumn{1}{r|}{\textbf{0.017}} & \textbf{0.921} & \multicolumn{1}{r|}{\textbf{0.024}} & \textbf{0.835} & \textbf{0.021} \\ \bottomrule
\end{tabular}
}
\end{table}

\begin{table}[!h]
\caption{Performance with and without the cosine randomized embedding spread features for a MalConv model with dropout added before each fully connected layer trained to detect malware using the EMBER2018 dataset. Best results are shown in \textbf{bold}.}
\label{tab:malware_std}
\adjustbox{max width=\columnwidth}{
\begin{tabular}{@{}lll|rrrrrrrrrrrr@{}}
\toprule
 &  & Num/Class & \multicolumn{1}{l}{n=100} &  &  &  & \multicolumn{1}{l}{n=50} &  &  &  & \multicolumn{1}{l}{n=25} &  &  &  \\
 &  & Metric & \multicolumn{2}{l}{AUC} & \multicolumn{2}{l}{Recall} & \multicolumn{2}{l}{AUC} & \multicolumn{2}{l}{Recall} & \multicolumn{2}{l}{AUC} & \multicolumn{2}{l}{Recall} \\
 &  & Statistic & avg & \multicolumn{1}{r|}{std} & avg & \multicolumn{1}{r|}{std} & avg & \multicolumn{1}{r|}{std} & avg & \multicolumn{1}{r|}{std} & avg & \multicolumn{1}{r|}{std} & avg & std \\
Experiment & OOD Model & Features &  & \multicolumn{1}{r|}{} &  & \multicolumn{1}{r|}{} &  & \multicolumn{1}{r|}{} &  & \multicolumn{1}{r|}{} &  & \multicolumn{1}{r|}{} &  &  \\ \midrule
EMBER2018 & LR & Last & 0.789 & \multicolumn{1}{r|}{0.007} & 0.704 & \multicolumn{1}{r|}{0.043} & \textbf{0.786} & \multicolumn{1}{r|}{\textbf{0.008}} & 0.682 & \multicolumn{1}{r|}{0.054} & \textbf{0.778} & \multicolumn{1}{r|}{\textbf{0.018}} & 0.650 & 0.066 \\
 &  & Last+Spread & \textbf{0.793} & \multicolumn{1}{r|}{\textbf{0.008}} & \textbf{0.718} & \multicolumn{1}{r|}{\textbf{0.042}} & 0.783 & \multicolumn{1}{r|}{0.013} & \textbf{0.689} & \multicolumn{1}{r|}{\textbf{0.067}} & 0.766 & \multicolumn{1}{r|}{0.027} & \textbf{0.658} & \textbf{0.080} \\ \cmidrule(l){2-15} 
 & RF & Last & 0.757 & \multicolumn{1}{r|}{0.011} & 0.735 & \multicolumn{1}{r|}{0.046} & 0.752 & \multicolumn{1}{r|}{0.015} & 0.727 & \multicolumn{1}{r|}{0.060} & 0.748 & \multicolumn{1}{r|}{0.023} & 0.714 & 0.079 \\
 &  & Last+Spread & \textbf{0.791} & \multicolumn{1}{r|}{\textbf{0.011}} & \textbf{0.784} & \multicolumn{1}{r|}{\textbf{0.045}} & \textbf{0.782} & \multicolumn{1}{r|}{\textbf{0.014}} & \textbf{0.764} & \multicolumn{1}{r|}{\textbf{0.057}} & \textbf{0.770} & \multicolumn{1}{r|}{\textbf{0.018}} & \textbf{0.743} & \textbf{0.084} \\ \midrule
Brazilian & LR & Last & 0.685 & \multicolumn{1}{r|}{0.007} & \textbf{0.645} & \multicolumn{1}{r|}{\textbf{0.054}} & 0.680 & \multicolumn{1}{r|}{0.010} & 0.607 & \multicolumn{1}{r|}{0.072} & 0.668 & \multicolumn{1}{r|}{0.042} & 0.584 & 0.078 \\
 &  & Last+Spread & \textbf{0.741} & \multicolumn{1}{r|}{\textbf{0.023}} & 0.620 & \multicolumn{1}{r|}{0.039} & \textbf{0.734} & \multicolumn{1}{r|}{\textbf{0.023}} & \textbf{0.617} & \multicolumn{1}{r|}{\textbf{0.049}} & \textbf{0.712} & \multicolumn{1}{r|}{\textbf{0.034}} & \textbf{0.605} & \textbf{0.063} \\ \cmidrule(l){2-15} 
 & RF & Last & 0.724 & \multicolumn{1}{r|}{0.016} & 0.693 & \multicolumn{1}{r|}{0.038} & 0.705 & \multicolumn{1}{r|}{0.023} & 0.674 & \multicolumn{1}{r|}{0.055} & 0.679 & \multicolumn{1}{r|}{0.035} & 0.652 & 0.081 \\
 &  & Last+Spread & \textbf{0.839} & \multicolumn{1}{r|}{\textbf{0.010}} & \textbf{0.797} & \multicolumn{1}{r|}{\textbf{0.034}} & \textbf{0.813} & \multicolumn{1}{r|}{\textbf{0.016}} & \textbf{0.772} & \multicolumn{1}{r|}{\textbf{0.054}} & \textbf{0.776} & \multicolumn{1}{r|}{\textbf{0.024}} & \textbf{0.736} & \textbf{0.083} \\ \bottomrule
\end{tabular}
}
\end{table}
\FloatBarrier
\subsection{Additional Spectral Normalization Results} \label{sec:appendix_sn}

We repeated the vision OOD data detection experiments from \autoref{sec:mnist_ood_experiment} on a spectral normalized dropout LeNet5 trained for 100 epochs. While spectral normalization is not required to see an improvement from the inclusion of cosine spread features, spectral normalization improves OOD detection performance consistently, as shown in \autoref{tab:mnist-ood-sn99_std} when compared to \autoref{tab:mnist-ood-reg99_std}.

\begin{table}[!h]
\caption{Performance with and without the cosine randomized embedding spread features for a dropout, spectral normalized LeNet5 trained on MNIST. Best results are shown in \textbf{bold}.}
\label{tab:mnist-ood-sn99_std}
\adjustbox{max width=\columnwidth}{
\begin{tabular}{llll|rrrrrrrrrrrr}
\hline
 &  &  & Num/Class & \multicolumn{4}{l}{n=1000} & \multicolumn{4}{l}{n=100} & \multicolumn{4}{l}{n=10} \\
 &  &  & Metric & \multicolumn{2}{l}{AUC} & \multicolumn{2}{l}{Acc} & \multicolumn{2}{l}{AUC} & \multicolumn{2}{l}{Acc} & \multicolumn{2}{l}{AUC} & \multicolumn{2}{l}{Acc} \\
 &  &  & Statistic & avg & \multicolumn{1}{r|}{std} & avg & \multicolumn{1}{r|}{std} & avg & \multicolumn{1}{r|}{std} & avg & \multicolumn{1}{r|}{std} & avg & \multicolumn{1}{r|}{std} & avg & std \\
OOD Train & OOD Test & OOD Model & Features &  & \multicolumn{1}{r|}{} &  & \multicolumn{1}{r|}{} &  & \multicolumn{1}{r|}{} &  & \multicolumn{1}{r|}{} &  & \multicolumn{1}{r|}{} &  &  \\ \hline
Fashion & Kuzushiji & LR & Last & 0.982 & \multicolumn{1}{r|}{0.003} & 0.921 & \multicolumn{1}{r|}{0.005} & 0.978 & \multicolumn{1}{r|}{0.016} & 0.921 & \multicolumn{1}{r|}{0.016} & 0.980 & \multicolumn{1}{r|}{0.040} & 0.908 & 0.032 \\
 &  &  & Last+Spread & \textbf{0.984} & \multicolumn{1}{r|}{\textbf{0.002}} & \textbf{0.926} & \multicolumn{1}{r|}{\textbf{0.006}} & \textbf{0.980} & \multicolumn{1}{r|}{\textbf{0.015}} & \textbf{0.926} & \multicolumn{1}{r|}{\textbf{0.019}} & \textbf{0.984} & \multicolumn{1}{r|}{\textbf{0.008}} & \textbf{0.915} & \textbf{0.024} \\ \cline{3-16} 
 &  & RF & Last & 0.979 & \multicolumn{1}{r|}{0.001} & 0.936 & \multicolumn{1}{r|}{0.004} & 0.972 & \multicolumn{1}{r|}{0.004} & 0.931 & \multicolumn{1}{r|}{0.010} & 0.963 & \multicolumn{1}{r|}{0.014} & 0.919 & 0.024 \\
 &  &  & Last+Spread & \textbf{0.985} & \multicolumn{1}{r|}{\textbf{0.001}} & \textbf{0.939} & \multicolumn{1}{r|}{\textbf{0.004}} & \textbf{0.982} & \multicolumn{1}{r|}{\textbf{0.002}} & \textbf{0.940} & \multicolumn{1}{r|}{\textbf{0.007}} & \textbf{0.981} & \multicolumn{1}{r|}{\textbf{0.008}} & \textbf{0.932} & \textbf{0.017} \\ \cline{2-16} 
 & notMNIST & LR & Last & \textbf{0.982} & \multicolumn{1}{r|}{\textbf{0.003}} & 0.916 & \multicolumn{1}{r|}{0.008} & 0.975 & \multicolumn{1}{r|}{0.025} & 0.915 & \multicolumn{1}{r|}{0.024} & 0.979 & \multicolumn{1}{r|}{0.049} & 0.905 & 0.036 \\
 &  &  & Last+Spread & 0.982 & \multicolumn{1}{r|}{0.002} & \textbf{0.922} & \multicolumn{1}{r|}{\textbf{0.009}} & \textbf{0.975} & \multicolumn{1}{r|}{\textbf{0.025}} & \textbf{0.922} & \multicolumn{1}{r|}{\textbf{0.026}} & \textbf{0.983} & \multicolumn{1}{r|}{\textbf{0.012}} & \textbf{0.913} & \textbf{0.025} \\ \cline{3-16} 
 &  & RF & Last & 0.978 & \multicolumn{1}{r|}{0.001} & 0.933 & \multicolumn{1}{r|}{0.005} & 0.971 & \multicolumn{1}{r|}{0.005} & 0.929 & \multicolumn{1}{r|}{0.012} & 0.962 & \multicolumn{1}{r|}{0.014} & 0.917 & 0.026 \\
 &  &  & Last+Spread & \textbf{0.984} & \multicolumn{1}{r|}{\textbf{0.001}} & \textbf{0.942} & \multicolumn{1}{r|}{\textbf{0.004}} & \textbf{0.981} & \multicolumn{1}{r|}{\textbf{0.003}} & \textbf{0.940} & \multicolumn{1}{r|}{\textbf{0.008}} & \textbf{0.980} & \multicolumn{1}{r|}{\textbf{0.007}} & \textbf{0.932} & \textbf{0.018} \\ \hline
Kuzushiji & Fashion & LR & Last & \textbf{0.988} & \multicolumn{1}{r|}{\textbf{0.000}} & 0.948 & \multicolumn{1}{r|}{0.002} & \textbf{0.987} & \multicolumn{1}{r|}{\textbf{0.003}} & 0.944 & \multicolumn{1}{r|}{0.008} & 0.979 & \multicolumn{1}{r|}{0.045} & 0.918 & 0.033 \\
 &  &  & Last+Spread & 0.987 & \multicolumn{1}{r|}{0.001} & \textbf{0.949} & \multicolumn{1}{r|}{\textbf{0.002}} & 0.985 & \multicolumn{1}{r|}{0.003} & \textbf{0.944} & \multicolumn{1}{r|}{\textbf{0.007}} & \textbf{0.981} & \multicolumn{1}{r|}{\textbf{0.041}} & \textbf{0.932} & \textbf{0.037} \\ \cline{3-16} 
 &  & RF & Last & 0.980 & \multicolumn{1}{r|}{0.001} & 0.946 & \multicolumn{1}{r|}{0.002} & 0.972 & \multicolumn{1}{r|}{0.004} & 0.938 & \multicolumn{1}{r|}{0.007} & 0.966 & \multicolumn{1}{r|}{0.009} & 0.930 & 0.021 \\
 &  &  & Last+Spread & \textbf{0.985} & \multicolumn{1}{r|}{\textbf{0.001}} & \textbf{0.950} & \multicolumn{1}{r|}{\textbf{0.002}} & \textbf{0.982} & \multicolumn{1}{r|}{\textbf{0.004}} & \textbf{0.947} & \multicolumn{1}{r|}{\textbf{0.003}} & \textbf{0.983} & \multicolumn{1}{r|}{\textbf{0.002}} & \textbf{0.943} & \textbf{0.007} \\ \cline{2-16} 
 & notMNIST & LR & Last & \textbf{0.986} & \multicolumn{1}{r|}{\textbf{0.000}} & 0.937 & \multicolumn{1}{r|}{0.001} & \textbf{0.985} & \multicolumn{1}{r|}{\textbf{0.002}} & 0.934 & \multicolumn{1}{r|}{0.007} & 0.979 & \multicolumn{1}{r|}{0.029} & 0.910 & 0.025 \\
 &  &  & Last+Spread & 0.986 & \multicolumn{1}{r|}{0.001} & \textbf{0.946} & \multicolumn{1}{r|}{\textbf{0.001}} & 0.984 & \multicolumn{1}{r|}{0.003} & \textbf{0.939} & \multicolumn{1}{r|}{\textbf{0.006}} & \textbf{0.983} & \multicolumn{1}{r|}{\textbf{0.017}} & \textbf{0.922} & \textbf{0.021} \\ \cline{3-16} 
 &  & RF & Last & 0.979 & \multicolumn{1}{r|}{0.001} & 0.941 & \multicolumn{1}{r|}{0.002} & 0.972 & \multicolumn{1}{r|}{0.004} & 0.933 & \multicolumn{1}{r|}{0.008} & 0.964 & \multicolumn{1}{r|}{0.011} & 0.924 & 0.023 \\
 &  &  & Last+Spread & \textbf{0.985} & \multicolumn{1}{r|}{\textbf{0.001}} & \textbf{0.947} & \multicolumn{1}{r|}{\textbf{0.001}} & \textbf{0.982} & \multicolumn{1}{r|}{\textbf{0.002}} & \textbf{0.944} & \multicolumn{1}{r|}{\textbf{0.003}} & \textbf{0.982} & \multicolumn{1}{r|}{\textbf{0.002}} & \textbf{0.937} & \textbf{0.010} \\ \hline
notMNIST & Fashion & LR & Last & 0.988 & \multicolumn{1}{r|}{0.001} & 0.946 & \multicolumn{1}{r|}{0.002} & \textbf{0.986} & \multicolumn{1}{r|}{\textbf{0.003}} & 0.941 & \multicolumn{1}{r|}{0.009} & 0.983 & \multicolumn{1}{r|}{0.018} & 0.916 & 0.027 \\
 &  &  & Last+Spread & \textbf{0.988} & \multicolumn{1}{r|}{\textbf{0.001}} & \textbf{0.951} & \multicolumn{1}{r|}{\textbf{0.002}} & 0.985 & \multicolumn{1}{r|}{0.006} & \textbf{0.944} & \multicolumn{1}{r|}{\textbf{0.010}} & \textbf{0.986} & \multicolumn{1}{r|}{\textbf{0.002}} & \textbf{0.932} & \textbf{0.016} \\ \cline{3-16} 
 &  & RF & Last & 0.980 & \multicolumn{1}{r|}{0.001} & 0.945 & \multicolumn{1}{r|}{0.002} & 0.971 & \multicolumn{1}{r|}{0.005} & 0.938 & \multicolumn{1}{r|}{0.007} & 0.966 & \multicolumn{1}{r|}{0.010} & 0.931 & 0.017 \\
 &  &  & Last+Spread & \textbf{0.987} & \multicolumn{1}{r|}{\textbf{0.001}} & \textbf{0.952} & \multicolumn{1}{r|}{\textbf{0.001}} & \textbf{0.984} & \multicolumn{1}{r|}{\textbf{0.002}} & \textbf{0.947} & \multicolumn{1}{r|}{\textbf{0.004}} & \textbf{0.983} & \multicolumn{1}{r|}{\textbf{0.004}} & \textbf{0.941} & \textbf{0.010} \\ \cline{2-16} 
 & Kuzushiji & LR & Last & 0.986 & \multicolumn{1}{r|}{0.000} & 0.936 & \multicolumn{1}{r|}{0.002} & 0.984 & \multicolumn{1}{r|}{0.003} & 0.934 & \multicolumn{1}{r|}{0.007} & 0.984 & \multicolumn{1}{r|}{0.002} & 0.911 & 0.023 \\
 &  &  & Last+Spread & \textbf{0.989} & \multicolumn{1}{r|}{\textbf{0.000}} & \textbf{0.947} & \multicolumn{1}{r|}{\textbf{0.001}} & \textbf{0.987} & \multicolumn{1}{r|}{\textbf{0.002}} & \textbf{0.942} & \multicolumn{1}{r|}{\textbf{0.006}} & \textbf{0.987} & \multicolumn{1}{r|}{\textbf{0.003}} & \textbf{0.923} & \textbf{0.018} \\ \cline{3-16} 
 &  & RF & Last & 0.980 & \multicolumn{1}{r|}{0.001} & 0.942 & \multicolumn{1}{r|}{0.002} & 0.972 & \multicolumn{1}{r|}{0.004} & 0.934 & \multicolumn{1}{r|}{0.007} & 0.963 & \multicolumn{1}{r|}{0.011} & 0.926 & 0.018 \\
 &  &  & Last+Spread & \textbf{0.987} & \multicolumn{1}{r|}{\textbf{0.000}} & \textbf{0.947} & \multicolumn{1}{r|}{\textbf{0.001}} & \textbf{0.984} & \multicolumn{1}{r|}{\textbf{0.002}} & \textbf{0.944} & \multicolumn{1}{r|}{\textbf{0.004}} & \textbf{0.983} & \multicolumn{1}{r|}{\textbf{0.005}} & \textbf{0.936} & \textbf{0.012} \\ \hline
\end{tabular}
}
\end{table}

\FloatBarrier
\subsection{Cosine Distance vs. Euclidean Distance for Unsupervised Embeddings} \label{sec:appendix_unsup}

We also investigated cosine distance versus Euclidean distance for measuring randomized embedding dispersion in the unsupervised setting. In particular, we investigated a stacked denoising autoencoder variant \cite{Vincent2008ExtractingAutoencoders, Vincent2010StackedCriterion} where all layers are trained at the same time instead of stage-wise, and dropout with a dropout probability $p=0.1$ is used as the corrupting process at each layer of the encoder. At test time, the dropout corruption is left on to generate randomized embeddings. The denoising autoencoder was trained on MNIST for 20 epochs with a batch size of 64 using the Adam optimizer with a learning rate of 0.001, the default recommended settings, and a weight decay of 0.01. Image inputs were flattened, and the encoder architecture consisted of 6 ReLU activated linear layers of output dimensions: 784, 400, 400, 120, 120, and 84. The decoder architecture is similar to the encoder architecture but in reverse order.

\autoref{figs-re/dae-line-euc.tex} and \autoref{figs-re/dae-line-cos.tex} show consistent results. Embedding dispersion as measured by Euclidean distance is related to mean norm in an identical manner across in distribution and OOD data.
While not as well separated as in the supervised setting, in distribution data has lower embedding dispersion as measured by cosine distance when compared to OOD data. 

\begin{figure}[tb]
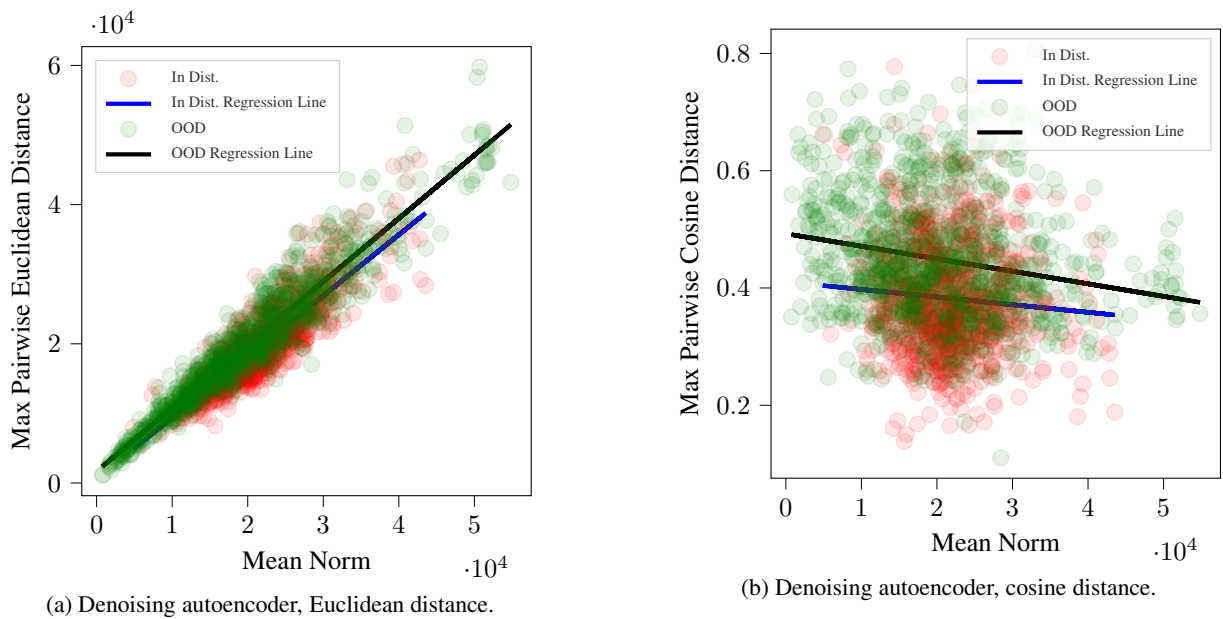

    \centering
    \begin{subfigure}[]{0.5\columnwidth}
        \centering
        \input{figs-re/dae-line-euc.tex}
        \caption{Denoising autoencoder, Euclidean distance.} \label{figs-re/dae-line-euc.tex}
    \end{subfigure}%
    ~
    \begin{subfigure}[]{0.5\columnwidth}
        \centering
        \input{figs-re/dae-line-cos.tex}
        \caption{Denoising autoencoder, cosine distance.} \label{figs-re/dae-line-cos.tex}
    \end{subfigure}%
    
    \caption{A comparison of the relationships between denoising autoencoder randomized embedding mean norm and the maximum pairwise distance for Euclidean distance and cosine distance respectively, for in distribution data (MNIST) and OOD data (Not-MNIST). Regression line fits are provided for each as well for easier comparison.}
    
\end{figure}
\FloatBarrier
\subsection{Simulations} 

\subsubsection{Mean and Variance of the Embedding Norms}
\label{sec:appendix_sim_mean_variance}
We perform a simulation to further illustrate the problem with the use of Euclidean distance in the case of a two layer ReLU activated network. 
As the depth of the BNN increases, the mean and variance of the embedding norms dramatically increase across layers, in particular as a consequence of the ReLU activation. This is known and bounds for this can be derived mathematically using the identity $max(x, 0) = 0.5(x + |x|)$ in the normal random matrix situation. However, we identify that the variance of the norms experiences a further increase due to the effect of dropout on preceding layers causing a carryover of variance into subsequent layers. Because dropout samples are taken across all layers simultaneously, the signal representing the distance between two embedding samples in layer $N$ is diluted with the inflated norm caused by preceding dropout in layers $1$ to $N-1$. This is confirmed by simulation on a two-layer neural network with dropout in \autoref{tab:appendix-mean-variance}, where the variance of the final embedding norms (4526.2) is much higher than it would be if dropout were only applied on that embedding layer (3124.0). This can explain why the Euclidean distance measure fails to perform for OOD detection.

\begin{table}[!h]
\centering
\caption{Mean and (variance) of the embedding norms in a simulated context.}
\label{tab:appendix-mean-variance}
\resizebox{0.8\textwidth}{!}{%
\begin{tabular}{@{}llll@{}}
\toprule
                       & Dropout only layer 1 & Dropout only layer 2 & Dropout both layers \\ \midrule
Layer 1 embedding norm & 96.0 (58.6)          & 118.6 (0.0)          & 96.0 (58.6)         \\
Layer 2 embedding norm & 599.7 (3328.0)       & 606.1 (3124.0)       & 501.0 (4526.2)      \\ \bottomrule
\end{tabular}%
}
\end{table}


\subsubsection{Correlation Analysis Between Measures of Uncertainty}
\label{sec:appendix_sim_corr}

To examine the relationships between the uncertainty features, we ran correlation analysis between all measures on the final embedding layer of a neural network, averaged over 1000 random matrix iterations. The embeddings form a $D \times B$ matrix, where $D$ is the embedding dimension and $B$ are the number of dropout samples, and we enforce a decaying correlation structure over the embedding dimensions. In \autoref{tab:appendix-corr-sim}, we summarize the correlations between all predictive features.

\begin{table}[!h]
\centering
\caption{Correlation analysis between measures of uncertainty in a simulated setting.}
\label{tab:appendix-corr-sim}
\resizebox{\textwidth}{!}{%
\begin{tabular}{@{}lllllll@{}}
\toprule
                 & mutual info.    & pred entr. & max softmax & max cos pdist & max euclid pdist & mean embed. norm \\ \midrule
mutual info.               & 1.00  & -0.31      & 0.23        & 0.08          & 0.32             & 0.51             \\
pred entr.       & -0.31 & 1.00       & -0.64       & 0.01          & -0.09            & -0.26            \\
max softmax      & 0.23  & -0.64      & 1.00        & 0.01          & 0.06             & 0.13             \\
max cos pdist    & 0.08  & 0.01       & 0.01        & 1.00          & 0.15             & -0.14            \\
max euclid pdist & 0.32  & -0.09      & 0.06        & 0.15          & 1.00             & 0.32             \\
mean embed. norm & 0.51  & -0.26      & 0.13        & -0.14         & 0.32             & 1.00             \\ \bottomrule
\end{tabular}%
}
\end{table}


This result indicates that the previously used features have higher inter-correlation than the max cosine pairwise distance, suggesting that our new feature adds an orthogonal measure of information that is not previously captured. This helps explain our improvement in OOD detection.

\FloatBarrier
\subsection{Additional Experiments on MNIST Variants}
\label{sec:appendix_more_mnist_experiments}

\subsubsection{Is Some OOD Training Data Needed?}
\label{sec:appendix_ood_training_data}

To compare with methods that do not require any OOD training data at all, we attempted the following where a linear kernel one class SVM and an Isolation Forest are used as outlier detectors that would hopefully capture OOD data. Results are shown in \autoref{tab:appendix-if-svm}. Generally, the best AUC is achieved using an Isolation Forest but the accuracy remains low. This is consistent with our conclusions that the relationship contains non-linear information and that some form of OOD data is needed to choose the appropriate threshold, and that as few as $n=10$ OOD points can estimate that threshold with significantly greater accuracy and AUC.  

\begin{table}[!h]
\centering
\caption{To compare with methods that do not require any OOD training data at all, we attempted the following where a linear kernel one class SVM and an Isolation Forest (IF) are used as outlier detectors.}
\label{tab:appendix-if-svm}
\resizebox{0.5\textwidth}{!}{%
\begin{tabular}{@{}lll|ll@{}}
\toprule
          &           & Metric      & AUC      & Acc      \\
OOD Test  & OOD Model & Features    &          &          \\ \midrule
Kuzushiji & SVM       & Last        & 0.555574 & 0.539760 \\
          &           & Last+Spread & 0.190757 & 0.253412 \\ \cmidrule(l){2-5} 
          & IF        & Last        & 0.876447 & 0.804457 \\
          &           & Last+Spread & 0.858214 & 0.617729 \\ \midrule
notMNIST  & SVM       & Last        & 0.532724 & 0.523312 \\
          &           & Last+Spread & 0.307177 & 0.335988 \\ \cmidrule(l){2-5} 
          & IF        & Last        & 0.842468 & 0.766183 \\
          &           & Last+Spread & 0.869251 & 0.631671 \\ \midrule
Fashion   & SVM       & Last        & 0.526127 & 0.514582 \\
          &           & Last+Spread & 0.212349 & 0.262806 \\ \cmidrule(l){2-5} 
          & IF        & Last        & 0.860657 & 0.791121 \\
          &           & Last+Spread & 0.883436 & 0.649500 \\ \bottomrule
\end{tabular}%
}
\end{table}


\subsubsection{Results when using Euclidean Randomized Embedding Maximum Spread Features}
\label{sec:appendix_euclidean_features_mnist}

To compare Euclidean distance features with cosine distance features, we ran experiments and found that cosine does empirically does better, as expected. In \autoref{tab:appendix-euc-mnist-comp} are the results for the MNIST experiments where the Spread features use Euclidean distance.

\begin{table}[!h]
\centering
\caption{MNIST variant results  when  using  Euclidean  randomized  embedding  maximum  spread  features.}
\label{tab:appendix-euc-mnist-comp}
\resizebox{\textwidth}{!}{%
\begin{tabular}{@{}llllllllll@{}}
\toprule
          &           &           & Num/Class                        & n=1000   &                               & n=100    &                               & n=10     &          \\
          &           &           & \multicolumn{1}{l|}{Metric}      & AUC      & \multicolumn{1}{l|}{Acc}      & AUC      & \multicolumn{1}{l|}{Acc}      & AUC      & Acc      \\
OOD Train & OOD Test  & OOD Model & \multicolumn{1}{l|}{Features}    &          & \multicolumn{1}{l|}{}         &          & \multicolumn{1}{l|}{}         &          &          \\ \midrule
Fashion   & Kuzushiji & LR        & \multicolumn{1}{l|}{Last}        & 0.970909 & \multicolumn{1}{l|}{0.915305} & 0.968269 & \multicolumn{1}{l|}{0.910754} & 0.967937 & 0.891701 \\
          &           &           & \multicolumn{1}{l|}{Last+Spread} & 0.959380 & \multicolumn{1}{l|}{0.906095} & 0.953365 & \multicolumn{1}{l|}{0.897221} & 0.941887 & 0.864347 \\ \cmidrule(l){3-10} 
          &           & RF        & \multicolumn{1}{l|}{Last}        & 0.961252 & \multicolumn{1}{l|}{0.916221} & 0.949761 & \multicolumn{1}{l|}{0.899854} & 0.946880 & 0.880030 \\
          &           &           & \multicolumn{1}{l|}{Last+Spread} & 0.958593 & \multicolumn{1}{l|}{0.896063} & 0.945695 & \multicolumn{1}{l|}{0.892397} & 0.943430 & 0.879735 \\ \cmidrule(l){2-10} 
          & notMNIST  & LR        & \multicolumn{1}{l|}{Last}        & 0.966394 & \multicolumn{1}{l|}{0.910832} & 0.965956 & \multicolumn{1}{l|}{0.910824} & 0.961222 & 0.882841 \\
          &           &           & \multicolumn{1}{l|}{Last+Spread} & 0.966064 & \multicolumn{1}{l|}{0.914921} & 0.955173 & \multicolumn{1}{l|}{0.901126} & 0.922676 & 0.839990 \\ \cmidrule(l){3-10} 
          &           & RF        & \multicolumn{1}{l|}{Last}        & 0.958357 & \multicolumn{1}{l|}{0.916837} & 0.947242 & \multicolumn{1}{l|}{0.898442} & 0.939856 & 0.876058 \\
          &           &           & \multicolumn{1}{l|}{Last+Spread} & 0.966978 & \multicolumn{1}{l|}{0.927221} & 0.953618 & \multicolumn{1}{l|}{0.910864} & 0.945218 & 0.882916 \\ \midrule
Kuzushiji & Fashion   & LR        & \multicolumn{1}{l|}{Last}        & 0.972502 & \multicolumn{1}{l|}{0.919816} & 0.971007 & \multicolumn{1}{l|}{0.917156} & 0.968018 & 0.900125 \\
          &           &           & \multicolumn{1}{l|}{Last+Spread} & 0.968836 & \multicolumn{1}{l|}{0.923116} & 0.962746 & \multicolumn{1}{l|}{0.916281} & 0.961164 & 0.898184 \\ \cmidrule(l){3-10} 
          &           & RF        & \multicolumn{1}{l|}{Last}        & 0.963407 & \multicolumn{1}{l|}{0.920547} & 0.953377 & \multicolumn{1}{l|}{0.903246} & 0.938179 & 0.885983 \\
          &           &           & \multicolumn{1}{l|}{Last+Spread} & 0.964453 & \multicolumn{1}{l|}{0.919921} & 0.955725 & \multicolumn{1}{l|}{0.908291} & 0.941819 & 0.894097 \\ \cmidrule(l){2-10} 
          & notMNIST  & LR        & \multicolumn{1}{l|}{Last}        & 0.967112 & \multicolumn{1}{l|}{0.914289} & 0.965904 & \multicolumn{1}{l|}{0.911111} & 0.960355 & 0.885173 \\
          &           &           & \multicolumn{1}{l|}{Last+Spread} & 0.975508 & \multicolumn{1}{l|}{0.925026} & 0.966202 & \multicolumn{1}{l|}{0.917970} & 0.938949 & 0.869555 \\ \cmidrule(l){3-10} 
          &           & RF        & \multicolumn{1}{l|}{Last}        & 0.959856 & \multicolumn{1}{l|}{0.920416} & 0.948595 & \multicolumn{1}{l|}{0.903442} & 0.928044 & 0.876868 \\
          &           &           & \multicolumn{1}{l|}{Last+Spread} & 0.968210 & \multicolumn{1}{l|}{0.925442} & 0.958487 & \multicolumn{1}{l|}{0.912819} & 0.934295 & 0.887989 \\ \midrule
notMNIST  & Fashion   & LR        & \multicolumn{1}{l|}{Last}        & 0.964710 & \multicolumn{1}{l|}{0.911000} & 0.947280 & \multicolumn{1}{l|}{0.902452} & 0.968372 & 0.894182 \\
          &           &           & \multicolumn{1}{l|}{Last+Spread} & 0.963373 & \multicolumn{1}{l|}{0.918600} & 0.949651 & \multicolumn{1}{l|}{0.909618} & 0.962364 & 0.887704 \\ \cmidrule(l){3-10} 
          &           & RF        & \multicolumn{1}{l|}{Last}        & 0.960297 & \multicolumn{1}{l|}{0.909847} & 0.955611 & \multicolumn{1}{l|}{0.903392} & 0.946035 & 0.888309 \\
          &           &           & \multicolumn{1}{l|}{Last+Spread} & 0.968320 & \multicolumn{1}{l|}{0.920963} & 0.963052 & \multicolumn{1}{l|}{0.906211} & 0.954943 & 0.906248 \\ \cmidrule(l){2-10} 
          & Kuzushiji & LR        & \multicolumn{1}{l|}{Last}        & 0.960289 & \multicolumn{1}{l|}{0.899979} & 0.938433 & \multicolumn{1}{l|}{0.888839} & 0.967569 & 0.887464 \\
          &           &           & \multicolumn{1}{l|}{Last+Spread} & 0.956063 & \multicolumn{1}{l|}{0.901247} & 0.938300 & \multicolumn{1}{l|}{0.890312} & 0.950506 & 0.865883 \\ \cmidrule(l){3-10} 
          &           & RF        & \multicolumn{1}{l|}{Last}        & 0.958770 & \multicolumn{1}{l|}{0.905458} & 0.953364 & \multicolumn{1}{l|}{0.902915} & 0.940991 & 0.884212 \\
          &           &           & \multicolumn{1}{l|}{Last+Spread} & 0.959689 & \multicolumn{1}{l|}{0.892047} & 0.957301 & \multicolumn{1}{l|}{0.895558} & 0.946487 & 0.895423 \\ \bottomrule
\end{tabular}%
}
\end{table}

\subsubsection{Classifier Feature Importances}
\label{sec:appendix_feature_importances}
To further understand the contribution of our cosine distance measure, we compute the mean and standard deviation of feature Gini importances for the random forest classifiers fit across our MNIST variant experiments. Results are shown in \autoref{fig:figs-re/feat_imp_rf.tex} and show that our spread based features are important with layer 3's spread having a Gini importance comparable to traditional features such as predictive entropy. 

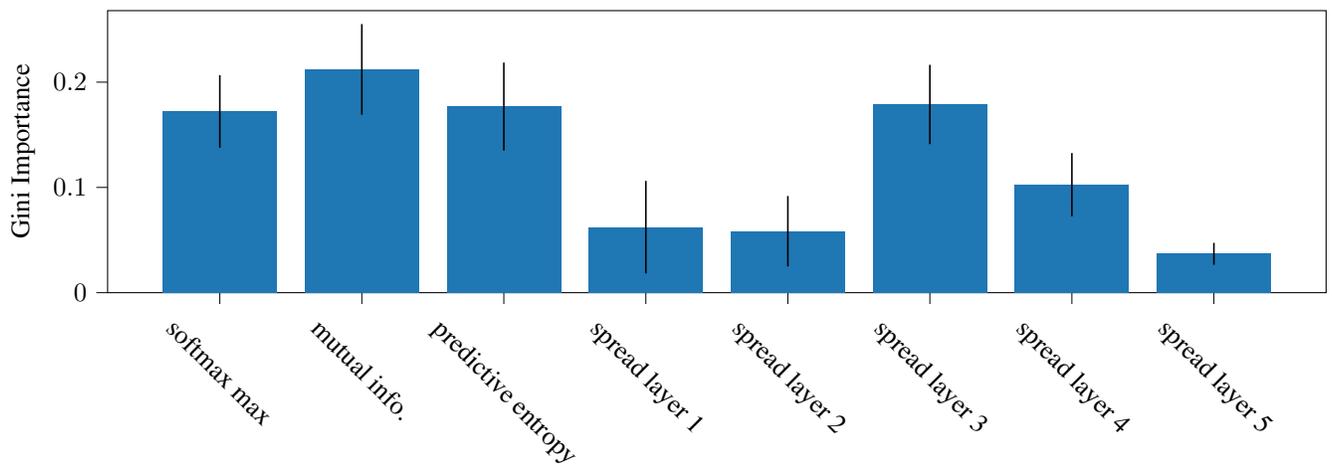
\begin{figure}[tb]
\centering
\begin{tikzpicture}

\definecolor{color0}{rgb}{0.12156862745098,0.466666666666667,0.705882352941177}

\begin{axis}[
tick align=outside,
tick pos=left,
x grid style={white!69.0196078431373!black},
xmin=-0.79, xmax=7.79,
xtick style={color=black},
xtick={0,1,2,3,4,5,6,7},
xticklabels={
  softmax max,
  mutual info.,
  predictive entropy,
  spread layer 1,
  spread layer 2,
  spread layer 3,
  spread layer 4,
  spread layer 5
},
y grid style={white!69.0196078431373!black},
ylabel={Gini Importance},
ymin=0, ymax=0.267942404912353,
ytick style={color=black},
xticklabel style={rotate=-45},
width=1\columnwidth,
height=0.3\columnwidth,
]
\draw[draw=none,fill=color0] (axis cs:-0.4,0) rectangle (axis cs:0.4,0.172114664362476);
\draw[draw=none,fill=color0] (axis cs:0.6,0) rectangle (axis cs:1.4,0.212115367815948);
\draw[draw=none,fill=color0] (axis cs:1.6,0) rectangle (axis cs:2.4,0.176790675776114);
\draw[draw=none,fill=color0] (axis cs:2.6,0) rectangle (axis cs:3.4,0.0622750019030351);
\draw[draw=none,fill=color0] (axis cs:3.6,0) rectangle (axis cs:4.4,0.058377025622042);
\draw[draw=none,fill=color0] (axis cs:4.6,0) rectangle (axis cs:5.4,0.178794588642765);
\draw[draw=none,fill=color0] (axis cs:5.6,0) rectangle (axis cs:6.4,0.102607885996188);
\draw[draw=none,fill=color0] (axis cs:6.6,0) rectangle (axis cs:7.4,0.0369247898814307);
\path [draw=black, semithick]
(axis cs:0,0.137680963198354)
--(axis cs:0,0.206548365526598);

\path [draw=black, semithick]
(axis cs:1,0.169047492858227)
--(axis cs:1,0.255183242773669);

\path [draw=black, semithick]
(axis cs:2,0.135004860871015)
--(axis cs:2,0.218576490681214);

\path [draw=black, semithick]
(axis cs:3,0.0183361876095684)
--(axis cs:3,0.106213816196502);

\path [draw=black, semithick]
(axis cs:4,0.02493046599539)
--(axis cs:4,0.091823585248694);

\path [draw=black, semithick]
(axis cs:5,0.141102111087326)
--(axis cs:5,0.216487066198204);

\path [draw=black, semithick]
(axis cs:6,0.0725684729151001)
--(axis cs:6,0.132647299077277);

\path [draw=black, semithick]
(axis cs:7,0.026520144792407)
--(axis cs:7,0.0473294349704543);

\end{axis}

\end{tikzpicture}
\caption{Random Forest Gini feature importances for MNIST variant experiments. Means and standard deviations are shown.} 
\label{fig:figs-re/feat_imp_rf.tex}
\end{figure}

\FloatBarrier
\subsection{Embedding Component Variance} 

In the context of a linear layer with input $\bf x$ indexed by $i$, output $\bf y$ indexed by $j$, weight matrix $W$, bias $\bf b$, dropout with probability $p$ of \textit{not} being dropped, the layer can be written as
\[
y_j = \left( \sum_i D_i \cdot W_{ij} \cdot x_i  \right) + b_j
\]
where $D_i \sim \textrm{Bern}(p)$ are i.i.d. Bernoulli random variables with probability parameter $p$. The variance of an embedding component can be written as follows:
\[
\textrm{Var}(y_j) = \textrm{Var} \left[ \left( \sum_i D_i \cdot W_{ij} \cdot x_i  \right) + b_j \right]
\]

Variance is invariant to changes in a location parameter, and the $D_i$ are i.i.d. allowing us to write:
\[
\textrm{Var}(y_j) =  \sum_i (W_{ij} \cdot x_i)^2 \: \textrm{Var}(D_i) =  \sum_i (W_{ij} \cdot x_i)^2 \: p \: (1-p)
\]

\FloatBarrier
\subsection{Dataset Links} 
\label{sec:appendix_datasets}

Data used in the image classification experiments can be found here:

\begin{itemize}
  \item \url{http://yann.lecun.com/exdb/mnist/}
  \item \url{https://github.com/davidflanagan/notMNIST-to-MNIST}
  \item \url{https://github.com/rois-codh/kmnist}
  \item \url{https://github.com/zalandoresearch/fashion-mnist}
\end{itemize}

Data used in the language classification experiments can be found here:
\url{https://zenodo.org/record/841984#.YK0r8S1h1pQ}

Part of the data used in the malware detection experiments can be found here:

\begin{itemize}
  \item \url{https://github.com/elastic/ember}
  \item \url{https://github.com/fabriciojoc/brazilian-malware-dataset}
\end{itemize}

The 1.1TB of raw PE files are not available as part of EMBER2018, but they can be downloaded via VirusTotal: 
\url{https://www.virustotal.com/gui/}

\end{document}